\documentclass[twoside,11pt]{article}

\usepackage{jmlr2e}

\usepackage{graphicx}
\usepackage{caption}
\usepackage{subcaption}
\usepackage{paralist}
\usepackage{url}
\usepackage[outdir=./]{epstopdf}
\usepackage{mathtools} 
\usepackage{verbatim}
\usepackage{floatrow}
\usepackage{tabularx}
\usepackage{algorithm2e}
\usepackage{comment}
\usepackage{tikz}
\usepackage{multirow}
\usepackage{enumitem}
\usepackage{xpatch}
\usepackage{lastpage}
\usetikzlibrary{matrix}
\usetikzlibrary{shapes.misc, positioning}
\newdimen\arrowsize
\tikzstyle{rect} = [rectangle, rounded corners, minimum width=3.5cm, minimum height=3cm,text centered, draw=black, fill=orange!30]
\tikzstyle{rect2} = [rectangle, rounded corners, minimum width=2.5cm, minimum height=0.9cm,text centered, draw=black, fill=green!30]
\tikzstyle{arrow} = [thick,->,>=stealth]

\pgfarrowsdeclare{arcsq}{arcsq}
{Co
  \arrowsize=0.2pt
  \advance\arrowsize by .5\pgflinewidth
  \pgfarrowsleftextend{-4\arrowsize-.5\pgflinewidth}
  \pgfarrowsrightextend{.5\pgflinewidth}
}
{
  \arrowsize=1.5pt
  \advance\arrowsize by .5\pgflinewidth
  \pgfsetdash{}{0pt} 
  \pgfsetroundjoin   
  \pgfsetroundcap    
  \pgfpathmoveto{\pgfpoint{0\arrowsize}{0\arrowsize}}
  \pgfpatharc{-90}{-160}{4\arrowsize}
  \pgfusepathqstroke
  \pgfpathmoveto{\pgfpointorigin}
  \pgfpatharc{90}{160}{4\arrowsize}
  \pgfusepathqstroke
}
\usetikzlibrary{shapes.symbols}

\DeclareMathOperator*{\argmin}{arg\,min}

\newcommand\mean{{\mathbb{E}}}
\newcommand\given{{\,|\,}}
\newcommand\BB[1]{{\mathbb{#1}}}
\newcommand\B[1]{{\mathbf{#1}}}
\newcommand\C[1]{{\mathcal{#1}}}

\newcommand{\PA}[2][]{{\B{PA}}^{#1}_{#2}}

\newfloatcommand{capbtabbox}{table}[][\FBwidth]
\newcommand{\independent}{\protect\mathpalette{\protect\independenT}{\perp}}
\def\independenT#1#2{\mathrel{\rlap{$#1#2$}\mkern2mu{#1#2}}}

\let\oldnl\nl
\newcommand{\nonl}{\renewcommand{\nl}{\let\nl\oldnl}}

\jmlrheading{19}{2018}{1-\pageref{LastPage}}{09/16; Revised 01/18}{09/18}{16-432}{Mateo Rojas-Carulla and Bernhard Sch\"olkopf and Richard Turner and Jonas Peters}

\ShortHeadings{Invariant Models for Causal Transfer Learning}{Rojas-Carulla and Sch\"olkopf and Turner and Peters}
\firstpageno{1}

\begin{document}

\title{Invariant Models for Causal Transfer Learning}

\author{\name Mateo Rojas-Carulla \email mr597@cam.ac.uk \\
  \addr Max Planck Institute for Intelligent Systems \\
  T\"ubingen, Germany\\[0.5\baselineskip]
  \addr Department of Engineering\\
  Univ.\ of Cambridge, United Kingdom
       \AND
       \name Bernhard Sch\"olkopf \email bs@tuebingen.mpg.de \\
       \addr Max Planck Institute for Intelligent Systems \\
       T\"ubingen, Germany
       \AND
       \name Richard Turner \email ret26@cam.ac.uk \\
       \addr Department of Engineering\\
       Univ.\ of Cambridge, United Kingdom
       \AND
       \name Jonas Peters\thanks{Most of this work was done while JP was at the Max Planck Institute for Intelligent Systems in T\"ubingen.} \email jonas.peters@math.ku.dk\\
       \addr Department of Mathematical Sciences\\
       Univ.\ of Copenhagen, Denmark
}

\editor{Massimiliano Pontil}

\maketitle

\begin{abstract}
Methods of 
transfer learning
try to combine knowledge from several related tasks (or domains) to improve performance on a test task.
Inspired by causal methodology, we relax the usual covariate shift assumption and assume that it holds true for a {\em subset} of predictor variables: the conditional distribution of the target variable given this subset of predictors is invariant over all tasks.
We show how this assumption can be motivated from ideas in the field of causality. We focus on the problem of Domain Generalization, in which no examples from the test task are observed. 
We prove that in an adversarial setting using this subset for prediction is optimal in Domain Generalization; we further provide examples, in which the tasks are sufficiently diverse and the estimator therefore outperforms pooling the data, even on average.  
If examples from the test task are available, we also provide a method to transfer knowledge from the training tasks and exploit all available features for prediction. 
However, we provide no guarantees for this method. 
We introduce a practical method which allows for automatic inference of the above subset 
and provide corresponding code. We present results on synthetic data sets and a gene deletion data set. 
\end{abstract}


\graphicspath{{figures/}}

\section{Introduction}{\label{sec:intro}}

Standard approaches to supervised learning assume that training and test data can be modeled as an i.i.d.\ sample from a distribution $\mathbb{P} := \mathbb{P}^{(\mathbf{X},Y)}$.  The inputs $\mathbf{X}$ are often vectorial, and the outputs~$Y$ may be labels (classification) or continuous values (regression). 
The i.i.d.\
setting is theoretically well understood and yields remarkable predictive accuracy in problems such as image classification, speech recognition and machine translation \citep[e.g.,][]{schmidhuber2015deep,krizhevsky2012imagenet}. 
However, many real world problems do not fit into this setting. The field of transfer learning attempts to address the scenario in which distributions may change between training and testing. We focus on two different problems within transfer learning: domain generalization and multi-task learning. We begin by describing these two problems,  
followed by a discussion of existing assumptions made to address the problem of knowledge transfer, as well as the new assumption we assay in this paper.

\subsection{Domain generalization and multi-task learning}
Assume that we want to predict a target $Y \in \BB{R}$ from some predictor variable $\B{X} \in \BB{R}^p$. 
Consider $D$ training (or source) tasks\footnote{In this work, we use the expression ``task'' and ``domain'' interchangeably.} $\mathbb{P}^1, \ldots,\mathbb{P}^D$
where each $\mathbb{P}^k$ 
represents a probability distribution generating data $(\B{X}^k,Y^k) \sim \mathbb{P}^k$.
At training time, we observe a sample $\left(\B{X}_i^k,Y_i^k \right)_{i=1}^{n_k}$ for each source task $k\in \{1,\ldots,D\}$; at test time, we want to predict the target values of an unlabeled sample from the task $T$ of interest.
We wish to learn a map $f:\BB{R}^p\rightarrow \BB{R}$ with small expected squared loss  
$\C{E}_{\BB{P}^T}(f) = \mathbb{E}_{(\B{X}^T,Y^T)\sim \mathbb{P}^{T}} (Y^T-f(\B{X}^T))^2$
on the test task $T$.

In domain generalization (DG) \cite[e.g.,][]{muandet2013domain}, we have $T=D+1$, that is, we are interested in using information from the source tasks in order to predict $Y^{D+1}$ from $\B{X}^{D+1}$ in a related yet unobserved test task $\mathbb{P}^{D+1}$. 
To beat simple baseline techniques, regularity conditions on the differences of the tasks are required.
Indeed, if the test task differs significantly from the source tasks, we may run into the problem of negative transfer \citep{pan_survey_2010} and DG becomes impossible \citep{ben2010impossibility}.

If examples from the test task are available during training \cite[e.g.,][]{pan_survey_2010,baxter_model_2000}, we refer to the problem as 
asymmetric multi-task learning (AMTL). 
If the objective is to improve performance in all the training tasks \cite[e.g.,][]{caruana_multitask_1997}, we call 
the problem symmetric multi-task learning (SMTL), see Table~\ref{tab:taxo} for a summary of these settings.
\begin{table}
\small
  \begin{center}
\[\arraycolsep=0.8pt\def\arraystretch{1.2}
\begin{tabular}{c||l|c}
method & training data from & test domain \\\hline 
\multirow{2}{*}{Domain Generalization (DG)}& $(\B{X}^1, {Y}^1), \ldots, (\B{X}^D, {Y}^D)$ & \multirow{2}{*}{$T:={D+1}$}\\
&$(\B{X}^1, {Y}^1), \ldots, (\B{X}^D, {Y}^D), \B{X}^{D+1}$&\\\hline
\multirow{2}{*}{Asymm. Multi-Task Learning (AMTL)}& $(\B{X}^1, {Y}^1), \ldots, (\B{X}^D, {Y}^D) $& \multirow{2}{*}{$T := D$}\\
& $(\B{X}^1, {Y}^1), \ldots, (\B{X}^D, {Y}^D), {\B{X}}^D $&\\\hline
\multirow{2}{*}{Symm. Multi-Task Learning (SMTL)}& $(\B{X}^1, {Y}^1), \ldots, (\B{X}^D, {Y}^D) $& \multirow{2}{*}{all}\\
& $(\B{X}^1, {Y}^1), \ldots, (\B{X}^D, {Y}^D), \B{X}^1, \ldots, \B{X}^D $& 
\end{tabular}
\]
\end{center}
\caption{Taxonomy for domain generalization (DG) and multi-task learning (AMTL and SMTL). Each problem can either be used without (first line) or with (second line) additional unlabeled data.}
\label{tab:taxo}
\end{table}
In multi-task learning (MTL), which includes both AMTL and SMTL, if infinitely many labeled data are available from the test task, it is impossible to beat a method that learns on the test task and ignores the training tasks. 

\subsection{Prior work} 
A first family of methods assumes that \textbf{covariate shift} holds \citep[e.g.,][]{quionero-candela_dataset_2009,schweikert_empirical_2009}. This states that for all $k\in\{1,\ldots,D,T\}$, the conditional distributions $Y^k\given \B{X}^k$ are \textbf{invariant} between tasks. Therefore, the differences in the joint distribution of $\B{X}^k$ and $Y^k$ originate from a difference in the marginal distribution of $\B{X}^k$. Under covariate shift, for instance, if an unlabeled sample from the test task is available at training in the DG setting, the training sample can be re-weighted via importance sampling \citep{gretton2009covariate, shimodaira2000improving, sugiyama2008direct} so that it becomes representative of the test task. 

Another line of work focuses on \textbf{sharing parameters} between tasks. This idea originates in the hierarchical Bayesian literature~\citep{bonilla2007multi,gao2008knowledge}. For instance, \citet{lawrence2004learning} introduce a model for MTL in which the mapping $f_k$ in each task $k\in\{1,\ldots,D,T\}$ is drawn independently from a common Gaussian Process (GP), and the likelihood of the latent functions depends on a shared parameter $\theta$. A similar approach is introduced by \citet{evgeniou2004regularized}: 
they consider 
an SVM with weight vector $w^k = w_0 + v^k$, where $w_0$ is shared across tasks and $v^k$ is task specific. 
This allows for tasks to be similar (in which case $v^k$ does not have a significant contribution to predictions) or quite different. \citet{daume_iii_frustratingly_2010} use a related approach for MTL when there is one source and one target task. Their method relies on the idea of augmented feature space, which they obtain using two features maps $\Phi^s(\B{X}^s) = (\B{X}^s, \B{X}^s, 0)$ for the source examples and $\Phi^t(\B{X}^t) = (\B{X}^t, 0, \B{X}^t)$ for the target examples. They then train a classifier using these augmented features. Moreover, they propose a way of using available unlabeled data from the target task at training. 

An alternative family of methods is based on learning a set of \textbf{common features} for all tasks~\citep{argyriou_multi-task_2006,romera-paredes_exploiting_2012,argyriou2007spectral, raina_self-taught_2007}. For instance, \citet{argyriou_multi-task_2006,argyriou2007spectral} propose to learn a set of low dimensional features shared between tasks using $L^1$ regularization, and then learn all tasks independently using these features. In \citet{raina_self-taught_2007}, the authors construct a similar set of features using $L^1$ regularization but make use of only unlabeled examples. \cite{Chen2012} proposes to build shared feature mappings which are robust to noise by using autoencoders. 

Finally, the assumption introduced in this paper is based on a \textbf{causal} view on domain adaptation and transfer.

\citet{ScholkopfJPSZMJ2012} relate multi-task learning with the independence between cause and mechanism. This notion is closely related to exogeneity \citep{Zhang2015TARK}, which roughly states that a causal mechanism mapping a cause $X$ to $Y$ should not depend on the distribution of $X$.
Additionally, \citet{zhang_domain_2013} consider the problem of target and conditional shift when the target variable is causal for the features. They 
assume that there exists a linear mapping between the covariates in different tasks, and the parameters of this mapping only depend on the distribution of the target variable. Moreover, \citet{zhang_multi-source_2015} argue that the availability of multiple domains is sufficient to drop this previous assumption when the distribution of $Y^k$ and the conditional $\B{X}^k\given Y^k$ change independently. The conditional in the test task can then be written as a linear mixture of the conditionals in the source domains. 
The concept of invariant conditionals and exogeneity can also be used for causal discovery \citep{peters_causal_2015,Zhang2015TARK,PetJanSch17}. 

\subsection{Contribution}
Taking into account causal knowledge, \textbf{our approach} to DG and MTL assumes that covariate shift holds only for a subset of the features. 
From the point of view of causal modeling \citep{pearl_causality:_2009}, assuming invariance of conditionals makes sense if the conditionals represent causal mechanisms
\citep[e.g.,][]{Hoover1990}, see Section~\ref{sec:causality} for details. 
Intuitively, we expect that a causal mechanism is a property of the physical world, and it does not depend on what we feed into it. If the input (which in this case coincides with the covariates) shifts, the mechanism should thus remain invariant \citep{Hoover1990,JanSch10,peters_causal_2015}. In the anticausal direction, however, a shift of the input usually leads to a changing conditional \citep{ScholkopfJPSZMJ2012}.
In practice, prediction problems are often not causal --- we should allow for the possibility that the set of predictors contains variables that are causal, anticausal, or confounded, i.e., statistically dependent variables without a directed causal link with the target variable. We thus expect that there is a \emph{subset} $S^*$ of predictors, referred to as an {\bf invariant set}, for which the covariate shift assumption holds true, i.e., the conditionals of output given predictor $Y^k \given \B{X}_{S^*}^k$  are invariant across $k\in \{1,\ldots,D, T\}$.
If $S^*$ is a strict subset of all predictors, this relaxes full covariate shift. 
We prove that knowing $S^*$ leads to robust properties for DG. 
Once an invariant set is known, traditional methods for covariate shift can be applied as a black box, see Figure~\ref{fig:diagram}. In the MTL setting, when labeled or unlabeled examples from the test task are available during training, we might not want to discard the features outside of $S^*$ for prediction. Hence, we also propose a method to leverage the knowledge of the invariant set $S^*$ and the available examples from the test task in order to outperform a method that learns only on the test task. 

Finally, note that in this work, we concentrate on the linear setting, keeping in mind that this has specific implications for covariate shift.

\subsection{Organization of the paper}
Section~\ref{sec:invCond} formally describes our approach and its underlying assumptions; 
in particular, we assume that 
an invariant set $S^*$ is known. For DG, we prove in Section~\ref{sec:TL} that predicting using only features in $S^*$ is optimal in an adversarial setting. Moreover, we present an example in which we compare our proposed estimator with pooling the training data, a standard technique for DG. In MTL, when additional labeled examples from $T$ are available, one might want to use all available features for prediction. Section~\ref{sec:MTL} provides a method to address this. We discuss a link to causal inference in Section~\ref{sec:causality}.
Often, an invariant set $S^*$ is not known a priori. Section~\ref{sec:learnInv} presents a method for inferring an invariant set from data. Section~\ref{sec:experiments} contains experiments on simulated and real data.

\section{Exploiting invariant conditional distributions in transfer learning}{\label{sec:invCond}}
Consider a transfer learning regression problem with source tasks $\mathbb{P}^1, \ldots,\mathbb{P}^D$, where $(\B{X}^k, Y^k)\sim \BB{P}^k$ for $k\in\{1,\ldots,D\}$.\footnote{We assume throughout this work the existence of densities and that random variables have finite variance.}
We now formulate our main assumptions.\\
\begin{compactitem}
\item[(A1)\,]
There exists a subset $S^* \subseteq \{1, \ldots, p\}$ of predictor variables such that 
\begin{equation} \label{eq:invpred}
Y^k\given\B{X}_{S^*}^k   \, \overset{d}{=} \,  Y^{k'}\given\B{X}_{S^*}^{k'}
\;\quad  \forall k, k' \in \{1, \ldots, D\}.
\end{equation}
We say that $S^*$ is an \textbf{invariant set} which leads to {invariant conditionals}. 
Here, 
$ \overset{d}{=}$ denotes equality in distribution.
\item[(A1')] This invariance also holds in the test task $T$, i.e., \eqref{eq:invpred} holds for all $k, k' \in \{1, \ldots, D, T\}$. 
\item[(A2)\,] The conditional distribution of $Y$ given an invariant set $S^*$ is linear: there exists $\alpha\in\BB{R}^{|S^*|}$ and a random variable $\epsilon$ such that for all $k\in\{1,\ldots,D\}$,
 $   [Y^k \given \B{X}_{S^*}^k=x]\overset{d}{=} \alpha^t x+\epsilon^k,
 $
  that is $Y^k = \alpha^t \B{X}_{S^*}^k + \epsilon^k$, with $\epsilon^k \independent \B{X}_{S^*}^k$ and for all $k\in\{1,\ldots,D\}$, $\epsilon^k  \overset{d}{=} \epsilon$.
\end{compactitem}
\vspace{0.5cm}
Assumption~(A1') is stronger than (A1) only in the DG setting, where, of course, (A1') and (A2) imply the linearity also in the test task $T$. 
While Assumption~(A1) is testable from training data, 
see Section~\ref{sec:learnInv}, (A1') is not. 
In covariate shift, one usually assumes that (A1') holds for the set of all features. Therefore,~(A1') 
is a weaker condition than covariate shift, see Figure~\ref{fig:diagram}. We regard this assumption as a building block that can be combined with any method for covariate shift, applied to the subset $S^*$.
It is known that it can be arbitrarily hard to exploit the assumption of covariate shift in practice \citep{ben2010impossibility}. In a general setting, for instance, assumptions about the support of the training distributions $\BB{P}^1, \ldots, \BB{P}^D$ and the test distribution $\BB{P}^T$ must be made for methods such as re-weighting to be expected to work \citep[e.g.,][]{gretton2009covariate}. 
The aim of our work is not to solve the full covariate shift problem, but to elucidate a relaxation of covariate shift in which it holds given only a subset of the features. 
We concentrate on linear relations (A2), which circumvents the issue of overlapping supports, for example. 

For the remainder of this section, we assume that we 
are given an invariant subset $S^*$ that 
satisfies (A1) and (A2). Note that we will also require (A1') for DG. In MTL, the invariance can be tested on the labeled data available 
from the test task, so (A1) and (A1') are equivalent. 

We show how the knowledge of $S^*$ can be exploited for the DG problem (Section~\ref{sec:TL}) and in the MTL case (Section~\ref{sec:MTL}). 
Here and below, we focus on linear regression using squared loss
\begin{equation} \label{eq:loss}
\C{E}_{\BB{P}^T}(\beta) = \mathbb{E}_{(\B{X}^T,Y^T)\sim \mathbb{P}^{T}} (Y^T - \beta^t\B{X}^T)^2
\end{equation}
(the superscript $T$ corresponds to the test task, not to be confused with the transpose, indicated by superscript~$t$). 
We denote by $\C{E}_{\BB{P}^{1},\ldots,\BB{P}^{D}}(\beta)$ the squared error averaged over the training tasks $k \in \{1, \ldots, D\}$.

\subsection{Domain generalization (DG): no labels from the test task} {\label{sec:TL}}
We first study the DG setting in which we receive no labeled examples from the test task during training time. Throughout this subsection, we assume that additionally to~(A1) and~(A2), assumption~(A1') holds. 
It is important to appreciate that (A1') is a strong assumption that is not testable on the training data: it is an assumption about the test task. We believe no nontrivial statement about DG is possible without an assumption of this type.

Now, we introduce our proposed estimator, which uses the conditional mean of the target variable given the invariant set in the training tasks. We prove that this estimator is optimal in an adversarial setting.

\begin{figure}
\begin{tikzpicture}
\node[draw, rounded rectangle, fill=blue!10, text depth = 0.9cm,minimum width=9cm, minimum height=2.8cm,align = center, xshift = -2cm] (main){\\(A1): $\exists S^*\subseteq\{1,\ldots,p\}:\,Y\given \B{X}_{S^*}$ invariant.};
\node[draw,rounded rectangle, fill=orange!20, yshift = -0.5cm] at (main.center) (cs){Covariate shift holds: $Y\given \B{X}_{\{1,\ldots,p\}}$ invariant.};
\node[draw, rounded rectangle, fill=green!20, right of=main, xshift =6.6cm, minimum height =1cm,align = center] (csm){Use methods for covariate \\shift, applied to $S^*$.\\ Here, (A2): \textbf{linear model}};
\draw [arrow] (main) -- (csm);
\end{tikzpicture}
\caption{Assumption (A1) (blue) is a relaxation of covariate shift (orange): the covariate shift assumption is a special case of (A1) with $S^* = \{1, \ldots, p\}$. Given the invariant set $S^*$, methods for covariate shift can be applied.}
\label{fig:diagram}
\end{figure}
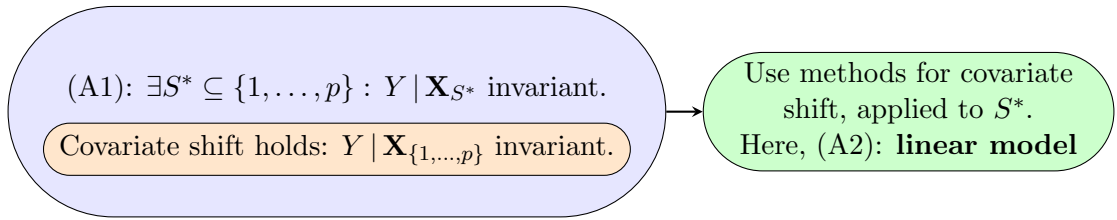
\paragraph{Proposed estimator.}
The optimal predictor obtained by minimizing ~\eqref{eq:loss}
is the conditional mean~
\begin{equation}
  \beta^{opt} := \argmin_{\beta\in\BB{R}^p} \C{E}_{\BB{P}^T}(\beta),
  \label{eq:betaopt}
\end{equation}
which is not available during training time.
Given an invariant set $S^*$ satisfying~(A1), (A1') and (A2), we propose to use the corresponding conditional expectation as an estimator. In other words, let $\beta^{S^*} = \argmin_{\beta\in\mathbb{R}^{|S^*|}} (Y^1-\beta^t\B{X}_{S^*}^1)^2$ be the vector obtained by minimizing the squared loss in the training tasks using only predictors in $S^*$. We propose as a predictor the vector $\beta^{CS(S^*)}\in\mathbb{R}^{p}$ obtained by adding zeros to $\beta^{S^*}$ in the dimensions corresponding to covariates outside of $S^*$. More formally, we propose to use as a predictor  
\begin{equation} \label{eq:estTL}
\begin{array}{ccc}
 \BB{R}^p & \rightarrow & \BB{R}\\
 \B{x} & \mapsto & \mathbb{E} [Y^1 \given \B{X}_{S^*}^1=\B{x}_{S^*}]
\end{array}
\quad \text{ and write } \quad 
\mathbb{E} [Y^1 \given \B{X}_{S^*}^1=\B{x}_{S^*}]  = \big(\beta^{CS(S^*)}\big)^t \, \B{x} 
.
\end{equation}
Because of~(A1), the conditional expectation in~\eqref{eq:estTL} is the same in all training tasks. 
In the limit of infinitely many data, given a subset $S$, $\beta^{CS(S)}$ is obtained by pooling the training tasks and regressing using only features in $S$. In particular, $\beta^{CS} := \beta^{CS(\{1,\ldots,p\})}$ is the estimator obtained when assuming traditional covariate shift.

\paragraph{Optimality in an adversarial setting.}
In an adversarial setting, predictor~\eqref{eq:estTL} satisfies the following optimality condition; as for the other results, the proof is provided in Appendix~\ref{app:proofs}. 
We state and prove a more general, nonlinear version of Theorem~\ref{prop:adversarial} in Appendix~\ref{proof:adversarial}.
\begin{theorem}[Adversarial] \label{prop:adversarial}
Consider $(\B{X}^1, Y^1) \sim \BB{P}^1$,$\ldots$, $(\B{X}^D, Y^D) \sim \BB{P}^D$ and an invariant set $S^*$ satisfying~(A1) and~(A2).
The proposed estimator satisfies an optimality statement over the set of distributions such that (A1') holds:
we have
$$
\beta^{CS(S^*)} \in \argmin_{\beta\in\BB{R}^p}  \sup_{\mathbb{P}^T \in \C{P}} \C{E}_{\BB{P}^T}(\beta), 
$$
where $\beta^{CS(S^*)}$ is defined in~\eqref{eq:estTL}
and $\C{P}$ 
contains all distributions over $(\B{X}^T, Y^T)$, $T= D+1$, that are absolutely continuous with respect to the same product measure $\mu$ and satisfy 
$Y^T \given \B{X}_{S^*}^T \overset{d}{=} Y^1 \given \B{X}_{S^*}^1$.
\end{theorem}
Unlike the optimal predictor~$\beta^{opt}$,  
the proposed estimator~\eqref{eq:estTL} can be learned from the data available in the training tasks. Given a sample $(\B{X}^k_1, Y^k_1), \ldots, (\B{X}^k_{n_k}, Y^k_{n_k})$ from tasks $k \in \{1, \ldots, D\}$, we can estimate the conditional mean in~\eqref{eq:estTL} by regressing $Y^k$ on $\B{X}_{S^*}^k$. Due to~(A1), we may also pool the data over the different tasks %
and use $$(\B{X}^1_1, Y^1_1), \ldots, (\B{X}^1_{n_1}, Y^1_{n_1}), (\B{X}^2_{1}, Y^2_{1}), \ldots, (\B{X}^D_{n_D}, Y^D_{n_D})$$ as a training sample for this regression.

One may also compare the proposed estimator with pooling the training tasks, a standard baseline in transfer learning which corresponds to assuming that usual covariate 
shift holds. Focusing on a specific example, Proposition~\ref{prop:threeNodes} in the following paragraph shows that when the test tasks become diverse, predicting using~\eqref{eq:estTL} outperforms pooling on average over all tasks.

\paragraph{Comparison against pooling the data.}
We proved that the proposed estimator~\eqref{eq:estTL} does well on an adversarial setting, in the sense that it minimizes the largest error on a task in $\C{P}$. The following result provides an example in which we can analytically compare the proposed estimator with the estimator obtained from pooling the training data, which is a benchmark in transfer learning. We prove that in this setting, the proposed estimator outperforms pooling the data on average over test tasks when the tasks become more diverse.

Let $\B{X}_{S^*}^k$ be a vector of independent Gaussian variables in task $k$. Let the target $Y^k$ satisfy 
\begin{equation}\label{eq:mod}
Y^k = \alpha^t \B{X}_{S^*}^k + \epsilon^k\,,
\end{equation}
where for each $k\in\{1,\ldots,D\}$, $\epsilon^k$ is Gaussian and independent of $\B{X}_{S^*}^k$. We have $\B{X}^k = (\B{X}_{S^*}^k, Z^k)$, where  
$$
Z^k = \gamma^k Y^k + \eta^k\,,
$$
for some $\gamma^k \in \BB{R}$ and where $\eta^k$ is Gaussian and independent of $Y^k$.\footnote{Using the notation introduced later in Section~\ref{sec:causality}, this corresponds to a Gaussian SEM with DAG shown in Fig.~\ref{sdag}.} Moreover, assume that the training tasks are balanced. 
We compare properties of estimator $\beta^{CS(S^*)}$ defined in Equation~\eqref{eq:estTL} against the least squares estimator obtained from pooling the training data. 
In this setting, the tasks differ in coefficients $\gamma^k$, which are randomly sampled. 
We prove that the squared loss averaged over unseen test tasks is always larger for the pooled approach, when coefficients $\gamma^k$ are centered around zero. In the case where they are centered around a non-zero mean, we prove that when the variance between tasks (in this case, for coefficients $\gamma^k$) becomes large enough, the invariant approach also outperforms pooling the data.

\begin{proposition}[Average performance]\label{prop:threeNodes}
Consider the model described previously. Moreover, assume that the tasks differ as follows: the coefficients $\gamma^1,\ldots,\gamma^D, \gamma^T = \gamma^{D+1}$ are i.i.d.\ with mean zero and variance $\Sigma^2>0$. 
The tasks do not differ elsewhere. In particular, the distribution of $X_{S^*}^k$ is the same for all tasks. 
Then the least squares predictor obtained from pooling the $D$ training tasks~$\beta^{CS} = (\beta_{S^*}^{CS}, \beta_Z^{CS})$ satisfies:
\begin{align}\label{eq:errors_ineq}
\BB{E}_{\gamma^T}\left(\C{E}_{\BB{P}^T}\left(\beta^{CS}\right)\right)\geq\BB{E}_{\gamma^T}\left(\C{E}_{\BB{P}^T}\left(\beta^{CS(S^*)}\right)\right) = \sigma^2. 
\end{align}
In particular, this implies the following:
\begin{align}\label{eq:errors_ineq_all}
\BB{E}_{\gamma^1,\ldots,\gamma^D,\gamma^T}\left(\C{E}_{\BB{P}^T}\left(\beta^{CS}\right)\right)\geq\BB{E}_{\gamma^1,\ldots,\gamma^D,\gamma^T}\left(\C{E}_{\BB{P}^T}\left(\beta^{CS(S^*)}\right)\right) = \sigma^2. 
\end{align}
Moreover, if the coefficients $\gamma^1,\ldots, \gamma^D, \gamma^T$ are i.i.d. with non-zero mean $\mu$, 
~\eqref{eq:errors_ineq} holds for fixed $\gamma^1,\ldots,\gamma^D$ if $\Sigma^2 \geq P(\mu)$, where $P$ is a polynomial in $\mu$, see Appendix~\ref{proof:threeNodes} for details.
\end{proposition}
The proof of Proposition~\ref{prop:threeNodes} can be found in Appendix~\ref{proof:threeNodes}.
Figure~\ref{fig:error_proba} 
\begin{figure}
\centering
\includegraphics[scale = 0.5]{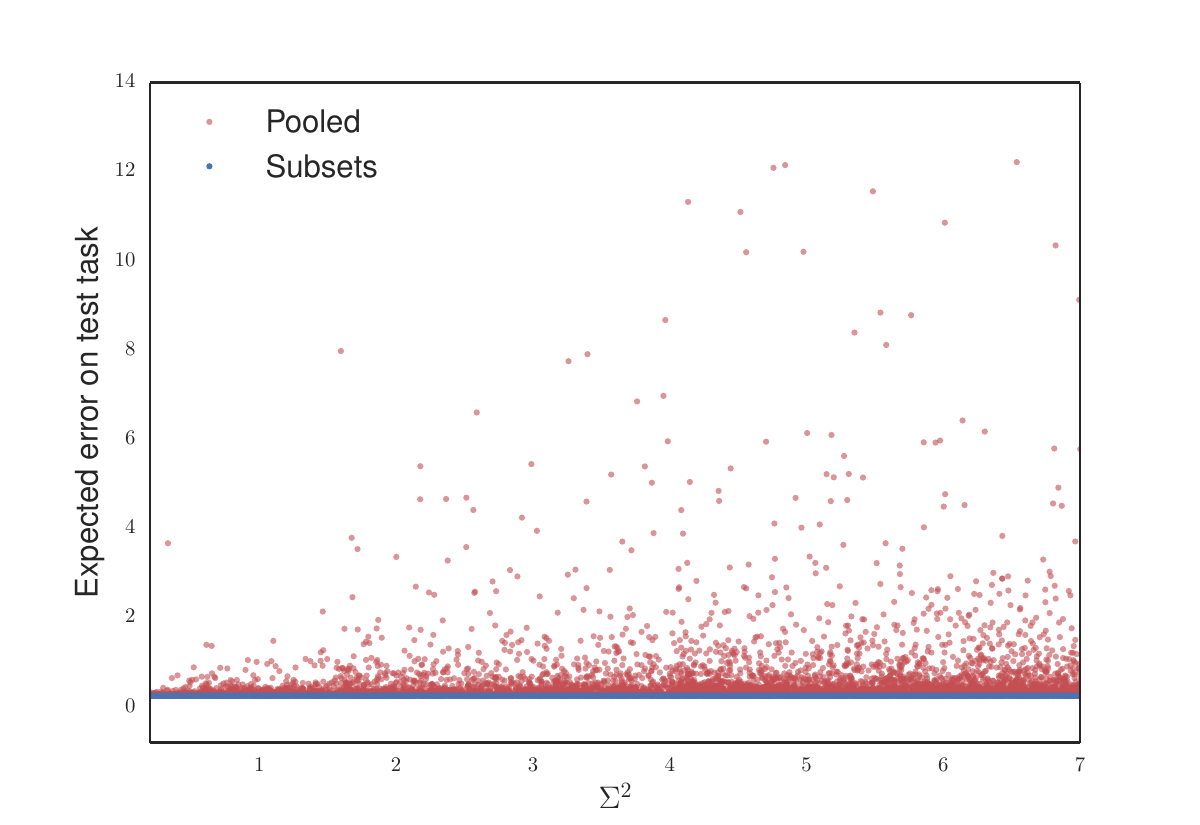}
\caption{
The figure shows expected errors for the pooled approach and the proposed method, see Equation~\eqref{eq:errors_ineq}.
$\mu = 0$. 
We consider two training tasks over $10,000$ simulations. In each, we randomly sample the variance of each covariate in $\B{X}$, the variance of $\eta$, and $\gamma$. $\sigma^2$ is the same in all tasks. 
As predicted by Proposition~Proposition~\ref{prop:threeNodes} observe that the error from the pooled approach (red) is systematically higher than the error from the prediction using only the invariant subset (blue), and both the error and its variance become large as the variance $\Sigma^2$ of coefficients $\gamma^k$ increases.}
\label{fig:error_proba}
\end{figure}
visualizes Proposition~\ref{prop:threeNodes} for two training tasks, it shows the expected errors 
for the pooled and invariant approaches, see~\eqref{eq:errors_ineq}, 
as the variance $\Sigma^2$ increases. Recall that $\Sigma^2$ corresponds to the variance of coefficients $\gamma^k$, and thus indicates how different the tasks are.
The expected errors are computed using the analytic expression found in the proof of Proposition~\ref{prop:threeNodes}.
As predicted by  Proposition~\ref{prop:threeNodes},
the expected error of the pooled approach
always exceeds the one of the proposed method (the
coefficients $\gamma^k$ are centered around zero),
see Equation~\eqref{eq:errors_ineq}. 
As $\Sigma^2$ tends to zero, $\gamma^k$ is close to zero in all tasks, which explains the equality of both the pooled and invariant errors for the limit case $\Sigma$ approaching $0$. For coefficients $\gamma^k$ centered around a non zero value, Equation~\eqref{eq:errors_ineq} does not necessarily hold for small $\Sigma^2$.

Proposition~\ref{prop:threeNodes} presents a setting in which the invariant approach outperforms pooling the data when the test errors are averaged over $\gamma$, i.e., $\BB{E}_{\gamma^T}\left(\C{E}_{\BB{P}^T}\left(\beta^{CS}\right)\right)\geq\BB{E}_{\gamma^T}\left(\C{E}_{\BB{P}^T}\left(\beta^{CS}\right)\right)$. It is also clear to see that the equality of the distribution of $\epsilon^k$ in Equation~\eqref{eq:mod} for all $k\in\{1,\ldots,D\}$ leads to $\mbox{Var}_{\gamma}\left(\C{E}_{\BB{P}^T}\left(\beta^{CS(S^*)}\right)\right) = 0$, thus our invariant estimator minimizes the variance of the test errors across all related tasks.  
 
\subsection{Multi-task learning (MTL): combining invariance and task-specific information}{\label{sec:MTL}}
In MTL, a labeled sample $\left( \B{X}_i^{T},Y_i^{T}\right)_{i=1}^{n_{T}}$ is available from the test task and the goal is to transfer knowledge from the training tasks. 
As before, we are given an invariant set $S^*$ satisfying
(A1) and~(A2).
Can we combine the invariance assumption with the new labeled sample and perform better than a method that trains only on the data in the test task?  
According to~(A1) and~(A2), the 
target satisfies $Y^k = \alpha^t\B{X}_{S^*}^k + \epsilon^k$, where the noise $\epsilon^k$ has zero mean and finite variance, is independent of $\B{X}_{S^*}^k$ and has the same distribution in the different tasks $k\in\{1,\ldots,D,T\}$. 
Our objective is to use the knowledge gained from the training tasks to get a better estimate of $\beta^{opt}$ defined in Equation~\eqref{eq:betaopt}. We describe below a way to tackle this using missing data methods.

\paragraph{Missing data approach}
In this section, we specify how we propose to tackle MTL by framing it as a missing data problem. While the idea is presented in the context of AMTL, it can be used for SMTL in the same way. In order to motivate the method, assume that for each $k\in\{1,\ldots,D,T\}$, there exists another probability distribution $\BB{Q}^k$ with density $q^k$ having the following properties:
(i) when restricted to $(\B{X}_{S^*}^k,Y^k)$, $\BB{Q}^k$ coincides with $\BB{P}^k$, (ii) the conditional $q^T(y \given \B{x}_{S^*}, \B{x}_N)$ coincides with $p^T(y\given \B{x}_{S^*}, \B{x}_N)$ on the test task and (iii) $q(y\given \B{x}_{S^*}, \B{x}_N) := q^k(y\given \B{x}_{S^*}, \B{x}_N)$ is the same in all tasks (which is not satisfied by $\BB{P}^k$, of course). 
The goal of learning the regression model from $Y$ on $\B{X}_{S^*}$ and $\B{X}_N$ in $\BB{P}^T$ coincides with the task of learning the same regression model in $\BB{Q}^T$. Property (iii) implies that we can pool the data from all tasks $\BB{Q}^k$. This is not possible, of course, for the given data, which we have received from the distributions $\BB{P}^k$. But now assume that in all training tasks, we only have access to the marginal $(\B{X}_{S^*}^k, Y^k)$ from $\BB{Q}^k$. Any method that addresses the regression under these constraints be used with the data available because of (i). We first prove the existence of such distributions $\BB{Q}^k$:

\begin{proposition}[Correctness of transfer]\label{prop:mtlok}
  Let $S^*$ be an invariant set verifying (A1) and (A2). For $k\in\{1,\ldots,D,T\}$, denote by $(\B{x},y) \mapsto p^k(\B{x},y)$ the density of $\BB{P}^k$. Then there exists a function $q:\BB{R}^p \rightarrow \BB{R}^+$ such that for each~$k\in\{1,\ldots,D,T\}$, there exists a distribution $\BB{Q}^k$ with density $q^k$ such that for all $(\B{x}, y)\in\BB{R}^{d+1}$, for all~$k\in\{1,\ldots,D,T\}$,
  \begin{compactitem}
  \item[i)] $ q^k(\B{x}_{S^*},y)  = p^k(\B{x}_{S^*}, y)$,
  \item[ii)] $q^T(y \given \B{x}_{S^*}, \B{x}_N) = p^T(y \given \B{x}_{S^*}, \B{x}_N)$,
  \item[iii)]$q^k(y \given \B{x}_{S^*}, \B{x}_N) = q(y \given \B{x}_{S^*}, \B{x}_N)$.
  \end{compactitem}
\end{proposition}
The proof for Proposition~\ref{prop:mtlok} can be found in Appendix~\ref{proof:mtlok}.
Following the previous intuition, for the training tasks $k\in\{1,\ldots,D\}$, we hide the data of $\B{X}^k_N$ and pretend the data in each task $k\in\{1,\ldots,D,T\}$ come from $\BB{Q}^k$. Note that some of the data are only missing for the training tasks. More precisely, $\B{X}_N^k$ is missing for $k\in\{1,\ldots,D\}$, while because of (i) in Proposition~\ref{prop:mtlok}, $(\B{X}_{S^*}^k,Y^k)$ is available for all tasks~$k\in\{1,\ldots,D,T\}$. We thus pool the data and learn a regression model of $Y$ versus $(\B{X}_{S^*}, \B{X}_N)$ by maximizing the likelihood of the observed data.

We formalize the problem as follows. Let $(\B{Z}_i)_{i=1}^n = (\B{X}_{S^*,i}, \B{X}_{N,i}, Y_i)_{i=1}^n$ be a pooled sample of the available data from the training tasks and the test task, in which $\B{X}_{N,i}$ is considered missing if $\B{X}_i$ is drawn from one of the training tasks. Here, $n=\sum_{k=1}^T n_k$ is the total number of training and test examples. Denote by $\B{Z}_{obs,i}$ the components of $\B{Z}_i$ which are not missing. In particular, $\B{Z}_{obs,i} = \B{Z}_i$ if $i$ is drawn from the test task and $\B{Z}_{obs,i} = (\B{X}_{S^*,i},Y_i)$ otherwise. Moreover, let $\Sigma$ be a $(p+1)\times(p+1)$ positive definite matrix, and $\Sigma_{i}$ is the submatrix of $\Sigma$ which corresponds to the observed features for example~$i$. If example $i$ is drawn from a training task, $\Sigma_i$ is of size $(|S^*|+1) \times (|S^*|+1)$, and $(p+1)\times (p+1)$ otherwise. The log-likelihood based on the observed data for matrix $\Sigma$ satisfies:

\begin{equation}
  \ell(\Sigma)  = \mbox{const} - \frac{1}{2}\sum_{i=1}^n\det\left(\Sigma_{i}\right)-\frac12 \B{Z}_{obs,i}^T\Sigma_{i}^{-1}\B{Z}_{obs,i},
  \label{eq:obsll}
\end{equation}
and our goal is to find $\Sigma$ which maximizes~\eqref{eq:obsll}. This model for the likelihood assumes that the data is multi-variate Gaussian with covariance matrix $\Sigma$.

When all data are observed, the least squares estimator~$\beta^{opt}$ can be seen as the result of a two step procedure.
First,~\eqref{eq:obsll} is maximized for the sample covariance matrix. Then, one computes the conditional mean $\mathbb{E}[Y\,|\,\B{X} = \B{x}]$ 
of the estimated joint distribution of ($\B{X}$,$Y$). In the case of missing data, however, the sample covariance matrix does no longer maximize~\eqref{eq:obsll}, see paragraph \textit{`A naive estimator for comparison'} below. Instead, we maximize~\eqref{eq:obsll} using EM.

Chapter 11 in~\citet{rubinmiss86} provides the update equations for optimizing Equation~\eqref{eq:obsll} using EM. More precisely, given an estimate $\Sigma^r$ of the covariance matrix at step $r$, the algorithm goes as follows.

\textbf{E step:} For an example $i$, we define

\[
  \B{Z}_i^r:=\begin{cases}
                \B{Z}_i \mbox{ if example }i\mbox{ is from the test task,}\\
                \left(\B{X}_{S^*,i}, \BB{E}(\B{X}_{N}^r\given \B{Z}_{obs,i}), Y_i\right)\mbox{ otherwise}.
            \end{cases}
\]

Here, we are essentially imputing the data for $\B{X}_N$ in the training tasks by the conditional mean given the observed data, using the current estimate of the covariance matrix $\Sigma^r$. The conditional expectation is computed using the current estimate $\Sigma^r$ and the Gaussian conditioning formula:
$$\BB{E}(\B{X}_{N}^r\given \B{Z}_{obs,i}) = \Sigma_{NZ_{obs}}^r(\Sigma_{Z_{obs}}^r)^{-1}\B{Z}_{obs,i},$$
where $\Sigma_{NZ_{obs}}^r$ is the submatrix of $\Sigma^r$ corresponding to the cross-covariance between $\B{X}_N$ and $(\B{X}_{S^*},Y)$, and $\Sigma_{Z_{obs}}^r$ is the submatrix corresponding to the covariance of $(\B{X}_{S^*},Y)$. For examples from the test task, we simply copy the example, since $\BB{P}^T = \BB{Q}^T$.
Moreover, define
\[
  C_{N,i}^r:=\begin{cases}
                0 \mbox{ if example }i\mbox{ is from the test task,}\\
                \mbox{Cov}(\B{X}_N^r\given \B{Z}_{obs,i}) = \Sigma_N^r-\Sigma_{NZ_{obs}}^r(\Sigma_{Z_{obs}}^r)^{-1}\Sigma_{Z_{obs}N}^r\mbox{ otherwise}.
            \end{cases}
\]

\textbf{M step:} compute the sample covariance given the imputed data:
$$
\Sigma^{r+1} = \frac1n \BB{E}\left(\sum_{i=1}^n \B{Z}_i^r (\B{Z}_i^r)^t \given \B{Z}_{obs,i},\Sigma^r \right) = \frac1n \sum_{i=1}^n \B{Z}_i^r (\B{Z}_i^r)^t + C_{i}^r,
$$
where $C_i^r$ is a $(p+1)\times (p+1)$ matrix whose submatrix corresponding to features in $N$ is $C_{N,i}^r$, and the remaining elements are $0$. The intuition for the M step is simple: we compute the sample covariance with the values imputed for $\B{X}_N$. Since these values are being imputed, matrix $C$ adds uncertainty for the corresponding values.

Once the algorithm has converged, we can read off the regression coefficient from the joint covariance matrix as $\BB{E}[Y \given \B{X}_{S^*} = \B{x}_{S^*}]$. The whole procedure is initialized with the sample covariance matrix computed with the available labeled sample from $T$.

\paragraph{Incorporating unlabeled data}
The previous method also allows us to incorporate unlabeled data from the test task. Indeed, assume that an unlabeled sample $\B{X}^T = (\B{X}_{S^*}^T, \B{X}_N^T)$ from the test task is also available at training time. This can be incorporated in the previous framework since the label $Y$ can be considered to be missing (as opposed to $\B{X}_N^T$ previously). We can then write $\B{Z}_i^r = (\B{X}_{S^*,i}, \B{X}_{N,i},\BB{E}(Y_i^r\given \B{Z}_{obs,i}))$ for the unlabeled data, thus imputing the value of $Y$ in in the E-step by the conditional mean given $(\B{X}_{S^*,i}, \B{X}_{N,i})$. The added covariance is then $C_{Y,i}^r = \mathrm{Var}(Y)^r-\Sigma_{YZ_{obs}}^r(\Sigma_{Z_{obs}}^r)^{-1}\Sigma_{Z_{obs}Y}^r$. The rest of the algorithm remains unchanged. 

\paragraph{A naive estimator for comparison}
In the population setting, Proposition~\ref{prop:combinelabels} in Appendix~\ref{sec:prop_combl} provides an expression for $\beta^{opt}$ as a function of $\alpha$ and $\epsilon$ from Assumption (A2). As in the previous paragraph, one could try to estimate the covariance matrix of $(\B{X},Y)$ using the knowledge of $\alpha$ and $\epsilon$ from the training tasks, and then read off the regression coefficients. 
In the presence of a finite amount of labeled and unlabeled data from the test task, a naive approach would thus plug in the knowledge of $\alpha$ and $\epsilon$ as follows:
the entries of
$\hat \Sigma_{\B{X},Y}$
that correspond to 
the covariances between 
$\B{X}_{S^*}$ and $Y$ are replaced with
$\hat \Sigma_{\B{X}_{S^*}} \cdot \alpha$, and the entry corresponding to the variance of $Y$ is replaced by $\alpha^t\hat \Sigma_{\B{X}_{S^*}}\alpha + \mathrm{Var}(\epsilon)$. 
This, however, often performs worse than forgetting about $\alpha$ and using the data in the test domain only, see Figure~\ref{fig:synt_numEx_mtl} (left). 
Why is this the case?
The naive solution described above leads to a matrix $\Sigma$ that does not only {\it not} maximize~\eqref{eq:obsll} but that often is not even positive definite. 
One needs to optimize over the free parameters of $\Sigma$, which corresponds to the covariance between $\B{X}_N$ and $Y$, given the constraint of positive definiteness. For comparison, we modified the naive approach as follows. First, we find a positive definite matrix satisfying the desired constraints. In order to do this, we solve a semi-definite Program (SDP) with a trivial objective which always equals zero. Then, we maximize the likelihood~\eqref{eq:obsll} over the free parameters of $\Sigma$ with a Nelder-Mead simplex algorithm. The constrained optimization problem can be shown to be convex in the neighborhood of the optimum \cite[][Sec. 3]{Zwiernik2014} if the number of data in the test domain grows. While gradients can be computed for this problem, gradient-based methods seem to perform poorly in practice (experiments are not shown for gradient based methods).

In an idealized scenario, infinite amount of unlabeled data in the test and labeled data in the training tasks could provide us with $\Sigma_{\B{X}}$, $\Sigma_{(\B{X}_{S^*},Y)}$ and $\mathrm{Var}(Y)$. We could then plug in these values into $\Sigma$ and optimize over the remaining parameters, see $\beta^{CS(cau+, i.d.)}$ in Figure~\ref{fig:synt_numEx_mtl} (left). 
In practice, we have to estimate $\Sigma_{\B{X}}$, $\Sigma_{(\B{X}_{S^*},Y)}$ and $\mathrm{Var}(Y)$ from data. Thus, the EM approach mentioned above constitutes the more principled approach.  

\subsection{Relation to causality} {\label{sec:causality}}
In this section, we provide a brief introduction to causal notions in order to motivate our method. More specifically, we show that under some conditions, the set $S^*$ of causal parents verifies Assumptions~(A1) and (A1').
Structural equation models (SEMs) \citep{pearl_causality:_2009} are one possibility to formalize causal statements. We say that a distribution over random variables $\B{X}  = (X_1, \ldots, X_p)$ is induced by a structural equation model with corresponding graph $\C{G}$ if each variable $X_j$ can be written as a deterministic function of its parents $\PA[\C{G}]{j}$ (in $\C{G}$) and some noise variable $N_j$:
\begin{equation} \label{eq:sem}
X_j = f_j(X_{\PA[\C{G}]{j}}, N_j)\,, \quad j = 1, \ldots, p\,.
\end{equation}
Here, the graph is required to be acyclic and the noise variables are assumed to be jointly independent. 
An SEM comes with the ability to describe \emph{interventions}. Intervening in the system corresponds to replacing one of the structural equations~\eqref{eq:sem}. The resulting joint distribution is called an intervention distribution. 
Changing the equation for variable $X_j$ usually affects the distribution of its children for example, but never the distribution of its parents. 
Consider now an SEM over variables $~(\B{X}, Y)$. Here, we do not specify the graphical relation between $Y$ and the other nodes: $Y$ may or may not have children or parents.
Suppose further that the different tasks $\BB{P}^1, \ldots, \BB{P}^D$ are intervention distributions of an underlying SEM with graph structure $\C{G}$. If the target variable has not been intervened on, then the set~$~S^*:= ~\PA[\C{G}]{Y}$ satisfies Assumptions~(A1) and~(A1'). 
This means that as long as the interventions will not take place at the target variable, the set $S^*$ of causal parents will satisfy Assumptions~(A1) and~(A1').

Recently, \citet{peters_causal_2015}
have given several sufficient conditions for the identifiability of the causal parents in the linear Gaussian framework. 
E.g., if the interventions take place at informative locations, or if we see sufficiently many different interventions, the set of causal parents is the \emph{only} set $S^*$ that satisfies Assumptions~(A1) and~(A1'). If there exists more than one set leading to invariant predictions, they consider the intersection of all such subsets. 
In this sense, seeing more environments helps for identifying the causal structure.
In this work, we are interested in prediction rather than causal discovery. Therefore, we try to find a trade-off between models that predict well and invariant models that generalize well to other domains. That is, in the DG setting, we are interested in the subset which leads to invariant conditionals and minimizes the prediction error across training tasks.

If the tasks $\mathbb{P}^k$ correspond to interventions in an SEM, we may construct an extended SEM with a parent-less environment variable $E$ that points into the intervened variables. Then, $\mathbb{P}^k$ equals the distribution of 
$(\B{X}, Y) \,|\,E = k$, see \cite[][Appendix~C]{peters_causal_2015}.
If the distribution of $(\B{X}, Y, E)$ is Markov and faithful w.r.t. the extended graph, the smallest set $S$ that leads to invariant conditionals and to best prediction is a subset of the Markov blanket of $Y$: certainly, it contains all parents of $Y$; if it includes a descendant of $Y$, this must be a child of $Y$ (which yields better prediction and still blocks any path from $Y$ to $E$); analogously, any contained ancestor of a child of $Y$ must be a parent of that child.

\section{Learning invariant conditionals}{\label{sec:learnInv}}
In the previous section, we have seen how a known invariant subset $S^*\subseteq \{1,\ldots,p\}$ of predictors leading to invariant conditionals $Y^k\given\B{X}_{S^*}^k$, see Assumptions~(A1) and (A1'), can be beneficial in the problems of DG and MTL. In practice, such a set $S^*$ is often unknown. 
We now present a method that aims at \textit{inferring} an invariant subset from data. Throughout this paper, we denote by $S$ any subset of features, while $S^+$ is an invariant set (which is not necessarily unique) for which (A1) holds. Such a subset $S^+$ \emph{does not necessarily satisfy both Assumptions (A1) and (A1')}. Indeed, in DG, only (A1) is testable in the training data. More precisely, if several invariant sets which satisfy (A1) are found, and only some of them satisfy (A1'), we cannot find these from data. We therefore have to add a criterion allowing us to select among several invariant sets. 
The method we propose provides an estimator $\hat{S}$ for an invariant subset $S^+$, which is chosen as the subset satisfying Assumption~(A1) which maximizes predictive accuracy on a validation set.  
In MTL, we still write $S^+$, even if we could then write $S^*$ as (A1') becomes testable. 
It is summarized in Algorithm~\ref{alg},
code is provided in \url{https://github.com/mrojascarulla/causal\_transfer\_learning}.

\subsection{Our method.}
\RestyleAlgo{boxruled}
\LinesNumbered
\begin{algorithm}[ht]
  \caption{Subset search\label{alg}}
  \nonl \textbf{Inputs:} Sample $(\B{x}_i^k,y_i^k)_{i=1}^{n_k}$ for tasks $k\in\{1,\ldots,D\}$, threshold $\delta$ for independence test. 

  \nonl \textbf{Outputs: } Estimated invariant subset $\hat{S}$.

Set $S_{acc} = \{\}$, $\mathrm{MSE} = \{\}$. 

\For{$S\subseteq \{1,\ldots,p\}$}{
  linearly regress $Y$ on $\B{X}_S$ and compute the residuals $R_{\beta^{CS(S)}}$ on a validation set. 

  compute $H = \mathrm{HSIC}_b\left((R_{\beta^{CS(S)},i},K_i)_{i=1}^{n}\right)$ and the corresponding p-value $p^*$ (or the p-value from an alternative test, e.g., Levene test.).

  \If{$p^* > \delta$}{ 

  compute $\widehat{\C{E}}_{\BB{P}^{1,\ldots,D}}(\beta^{CS(S)})$, the empirical estimate of $\C{E}_{\BB{P}^{1,\ldots,D}}(\beta^{CS(S)})$ on a validation set. 

  ${S}_{acc}.\mathrm{add}(S)$,
    $\mathrm{MSE}.\mathrm{add}(\widehat{\C{E}}_{\BB{P}^{1,\ldots,D}}(\beta^{CS(S)}))$} 
}
  Select $\hat S$ according to \textit{RULE}, see Section~\ref{sec:cvMTL}.
\end{algorithm}
\RestyleAlgo{boxruled}
\LinesNumbered
\begin{algorithm}[ht]
  \caption{Greedy subset search\label{alg:greedy}}
  \nonl \textbf{Inputs:} Sample $(\B{x}_i^k,y_i^k)_{i=1}^{n_k}$ for tasks $k\in\{1,\ldots,D\}$, threshold $\delta$ for independence test. 

  \nonl \textbf{Outputs: } Estimated invariant set $\hat{S}{greedy}$.

  Set $S_{acc} = \{\}$, $\hat{S}_{current}\{\}$, $\mathrm{MSE} = \{\}$. 

  \For{$i\in\{1,\ldots,n_{iters}\}$}{
    Set $stat_{min} = \infty$.
    
  \For{$S\in\C{S}_{\hat{S}_{current}}$}{
    
  linearly regress $Y$ on $\B{X}_S$ and compute the residuals $R_{\beta^{CS(S)}}$ on a validation set. 

  compute $H = \mathrm{HSIC}_b\left((R_{\beta^{CS(S)},i},K_i)_{i=1}^{n}\right)$ and the corresponding p-value $p^*$ (or the p-value from an alternative test, e.g. Levene test.).

\If{$p^* > \delta$}{
  compute $\widehat{\C{E}}_{\BB{P}^{1,\ldots,D}}(\beta^{CS(S)})$, the empirical estimate of $\C{E}_{\BB{P}^{1,\ldots,D}}(\beta^{CS(S)})$ on a validation set.

  $S_{acc}.\mathrm{add}(S)$,  $\mathrm{MSE}.\mathrm{add}(\widehat{\C{E}}_{\BB{P}^{1,\ldots,D}}(\beta^{CS(S)}))$, 
  
  set $\hat{S}_{current} = S$.
}
\ElseIf{$H<stat_{min}$}{
  set $\hat{S}_{current} = S$, $stat_{min}=H$.
  }
}
 }
 Select $\hat S$ according to \textit{RULE}, see Section~\ref{sec:cvMTL}.   
\end{algorithm}
Consider a set of $D$ tasks, a target variable $Y^k$ and a vector $\B{X}^k$ of~$~p$ predictor variables in task~$k$. 
For $\beta\in\BB{R}^p$, we define the residual in task $k$ as:
\begin{equation}
R^k_{\beta} =  Y^k- \beta^t \B{X}^k, \quad k\in\{1,\ldots,D\}.
\label{eq:residuals}
\end{equation} 
By Assumptions~(A1) and~(A2), there exists a subset $S^+$ and some vector $\beta^{CS(S^+)}$ such that for all $j\notin S^+$, $\beta_j^{CS(S^+)} = 0$ and $R_{\beta^{CS(S^+)}}^1 \overset{d}{=} \ldots \overset{d}{=} R_{\beta^{CS(S^+)}}^D$. Such a set $S^+$ is not necessarily unique. As stated in \citep{peters_causal_2015}, the number of invariant subsets decreases as more different tasks are observed at training time.
We propose to do an exhaustive search over subsets $S$ of predictors and statistically test for equality of the distribution of the residuals in the training tasks, see the section below. Among the accepted subsets, we select the subset $\hat{S}$ which leads to the smallest error on a validation set. This is a fundamental difference to the method proposed by \citet{peters_causal_2015}. Indeed, while our method addresses the transfer problem, \citet{peters_causal_2015} is about causal discovery. Algorithm~\ref{alg} finds an invariant subset which also leads to the lowest validation error. This subset may contain covariates which are non causal, see Section~\ref{sec:inform} for further details. On the other hand,~\citet{peters_causal_2015} estimate the causal parents (with coverage guarantee). Such an approach has a different purpose and performs very badly both in DG and MTL: e.g., when all tasks are identical, it uses the empty set as predictors, while our method selects the full set of predictors. 

In Section~\ref{subsec:scala}, we propose two solutions for when the number of predictors $p$ is too large for an exhaustive search: a greedy method and variable selection. 
While the algorithms are presented using linear regression, the extension to a nonlinear framework is straightforward. In particular, linear regression can be replaced by a nonlinear regression method.

\subsection{Statistical tests for equality of distributions.}
In order to test whether a subset $S$ leads to invariant conditionals, we can use a statistical test to check whether the residuals $R^k_{\beta^{CS(S)}}$ have the same distribution in all tasks $k\in\{1,\ldots,D\}$. We propose two possible methods. 

For Gaussian data, one can use a Levene test~\citep{levene1960robust} to test whether the residuals have the same variance in all tasks; their means are zero as long as an intercept is included in the regression model.

As an alternative, we propose a nonparametric $D$-sample test by testing whether the residuals are independent of the task index. This test is a direct application of HSIC~\citep{gretton2007kernel} 
but to our knowledge, is novel. 
Suppose that the index of the task can be considered as a random variable $K$. We consider the sample $Z = (R_{\beta^S,i},K_i)_{i=1}^{n}$ as drawn from a joint distribution over residuals and task indices, where $n=\sum_{k=1}^D n_k$ and $K_i\in\{1,\ldots,D\}$ is a discrete value indicating the index of the corresponding task. The residuals have the same distribution in all training tasks if and only if $R_{\beta^S}$ and $K$ are independent.  
Two characteristic kernels are used: a kernel $\kappa$ is used for embedding the residuals and a trivial kernel $d$ such that $d(i,j) = \delta_{ij}$ is used for $K$. 
Let therefore $\mathrm{HSIC}(R_{\beta^S},K)$ denote the value of the HSIC~\citep{gretton2007kernel} between $R_{\beta^S}$ and $K$, and let $\mathrm{HSIC}_b(Z)$ be the corresponding test statistic. A subset $S$ is accepted if if leads to accepting the null hypothesis of independence between $R_{\beta^S}$ and $K$ at level $\delta$. 

Both in the case of the Levene test and the $D$-sample test, 
the test outputs a p-value $p^*$, and we accept the null $H_0$ if $p^*>\delta$. Among these accepted subsets, we output the set $\hat{S}$ which leads to the smallest loss on a validation set. The test level $\delta$ is given as an input to our method and allows for a trade-off between predictive accuracy and exploiting invariance. As $\delta$ tends to zero, the null is accepted for all subsets
and we then select all features, which is equivalent to covariate shift. When $\delta$ approaches one, no subset is accepted as invariant. Our method then reduces to the mean prediction.  In order to compute p-values, a Gamma approximation is used for the distribution of $\mathrm{HSIC}_b(Z)$ under the null.

For non-additive models, one may even apply a conditional independence test~\citep[e.g.,][]{zhang_kernel-based_2011, fuku_kernel_dependence_2008} to test whether $K$ is independent of $Y\given \mathbf{X}_{S}$.

\subsection{Scalability to a large number of predictors} \label{subsec:scala}
When the number of features $p$ is large, full subset search is computationally not feasible. We propose two solutions for this scenario. 
If one has reasons to believe that the signal is sparse, that is the true set $S^*$ is small, one may use a {\bf variable selection} technique such as the Lasso~\citep{TibshiraniLasso}  as a first step.
Under the assumptions described in Section~\ref{sec:causality}, we know that the invariant set with the best prediction in the training tasks can be assumed to be a subset of the relevant features 
(which here equals the Markov blanket of $Y$).
Thus, if variable screening is satisfied ,i.e., one selects all relevant variables and possibly more,
the pre-selection step does not change the result of Algorithm~1 in the limit of infinitely many data. 
For linear models with $\ell_1$ penalization, variable screening is a well studied problem, see, e.g., compatibility and $\beta_{min}$ conditions \cite[][Chapter 2.5]{Buhlmann2011book}.

Alternatively, one may perform a {\bf greedy search} over subsets when full subset search is not feasible. 
Denote by $\C{S}_S$ the collection of neighboring sets of a set $S$ obtained by adding or removing exactly one predictor in $S$. If no subset has been accepted at a given iteration, we select the neighbor leading to the smallest test statistic. 
If a neighbor is accepted, we select the one which leads to the smallest training error. We start with the $p$ subsets with only one element, and allow to add or remove a single predictor at each step, see Algorithm~\ref{alg:greedy}. As often for greedy methods, there is no theoretical guarantee. 

\begin{figure}
\begin{tikzpicture}[scale=1.2, yscale=0.7, line width=0.5pt, inner sep=0.3mm, shorten >=1pt, shorten <=1pt]
\scriptsize
  \draw (1,2.5) node(x1) [circle, draw] {$X_1$};
  \draw (0.3,2) node(x2) [circle, draw] {$X_2$};
  \draw (1,1.5) node(x3) [circle, draw] {$X_3$};
  \draw (2,2) node(y) [circle, draw] {$\;Y\;$};
  \draw (3.7,2) node(x5) [circle, draw] {$X_5$};
  \draw (4,1) node(xw) [circle] {$ $};
  \draw (0,1) node(xe) [circle] {$ $};
  \draw[-arcsq] (x1) -- (y);
  \draw[-arcsq] (x2) -- (y);
  \draw[-arcsq] (x3) -- (y);
  \draw[-arcsq] (y) -- (x5);
\end{tikzpicture}
\caption{Example of a directed acyclic graph, see Section~\ref{sec:causality}. If $Y$ is not intervened on, the conditional $Y\given X_1, X_2, X_3$ remains invariant. 
\label{sdag}}
\end{figure}
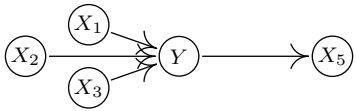
\subsection{Subset selection in MTL}\label{sec:cvMTL}
In DG, among the accepted subsets, we select the set $\hat S$ which leads to the lowest validation error. In MTL, however, a labeled sample from the test task $T$ is available at training time. Therefore, Algorithm~\ref{alg} is slightly modified. First, we get all the sets for which $H_0$ is accepted. Then, we select the accepted set $\hat S$ which leads to the smallest 5 fold cross validation error. For each subset, we compute the least squares coefficients using the procedure described in Section~\ref{sec:MTL}, and measure the prediction error on the held out validation set. Using the notation of Algorithm~\ref{alg}, let $S_{acc}$ be the set of subsets accepted as invariant, and let $MSE$ be the set of their corresponding squared errors on the validation set. The following rules are used for selecting an invariant set in DG and MTL.
\begin{compactitem}
\item[i)] \textbf{\textit{RULE} for DG}: Return $\hat S = S_{acc}[\argmin \mathrm{MSE}]$.
\item[ii)] \textbf{\textit{RULE} for MTL}: Define $CV_{acc} = \{\}$. For each set $S\subseteq S_{acc}$, do $CV_{acc}.\mathrm{add}(CV_S)$, where $CV_S$ is the 5-fold cross validation error over the labeled test data obtained by optimizing~\eqref{eq:obsll} using EM with subset $S$.

Return $\hat S = S_{acc}[\argmin CV_{acc}]$.
\end{compactitem}

Given a set of $k\in\{1,\ldots,T\}$ training tasks, a collection of sets $\hat{S}_1,\ldots,\hat{S}_u$ 
(eventually empty) is obtained, all of which lead to accepting the null hypothesis of invariance 
between the training tasks in DG.  
Our methods use the MSE on a validation set as a criterion for selecting a subset among these $u$ 
candidates. This is a design choice which is dependent on the specific application, and can be modified. 
For instance, if being conservative is important, the MSE may be an inappropriate choice. 
One may be then interested in combining confidence intervals for the accepted sets. 
One idea is to consider all accepted sets at the same time, one of which is, with probability $1-\alpha$, the set $S^*$ from Assumption (A1'). These sets yield different predictions, one of which stems from $S^*$, again, with probability $1-\alpha$. In some settings, it might be helpful to output the whole set of predictions. 
If one is interested in confidence intervals, these may be combined by taking its union. 
  \citet{heinze16} discuss this idea in the context of prediction under interventions. 

\begin{table}
\small
\begin{tabularx}{\columnwidth}{l|X}
estimator &  description\\\hline
$\beta^{CS(cau)}$ & Linear regr.\ with true causal predictors (often unknown in practice). \\
$\beta^{CS(\hat{S})}$& Finding the invariant set $\hat{S}$ using full subset search and performing lin.\ regr.\ using predictors in $\hat{S}$. $\hat{S}greedy$ corresponds to finding the invariant set using a greedy procedure. $\hat{S}Lasso$ corresponds to doing variable selection using Lasso as a first step, then doing full subset search on the selected features.\\
$\beta^{CS}$ & Pooling the training data and using linear regr.\\
$\beta^{CS(\hat{S}+)}$& Finding the invariant set $\hat{S}$ using full subset search and solve the optimization problem described in 'A naive estimator for comparison'.\\
$\beta^{CS(\hat{S}\sharp)}$& Finding the invariant set $\hat{S}$ using full subset search and maximizing~\eqref{eq:obsll} for MTL using EM.\\
$\beta^{mean}$ & Pooling the training data and outputting the mean of the target.\\
$\beta^{dom}$ & Linear regression using only the available labeled sample from $T$.\\
$\beta^{MTL}$ & Multi-task feature learning estimator~\citep{argyriou_multi-task_2006}.\\
$\beta^{DICA}$ & DICA~\citep{muandet2013domain} with rbf kernel. \\
$\beta^{mDA}$ & Pooling the training data and an unlabeled sample from $T$, learning features using mSDA~\citep{Chen2012} with one layer and linear output, then using linear regr.
\end{tabularx}%
\caption{Estimators used in the numerical experiments. A '+' next to a subset $S$ corresponds to the method for MTL described in the last paragraph of Section~\ref{sec:MTL}.}
\label{tab:est}
\end{table}

\section{Experiments}{\label{sec:experiments}}
We compare our estimator to different methods, which are summarized in
Table~\ref{tab:est}. 
$\beta^{CS(cau)}$ uses the ground truth for $S^*$ when it is available, $\beta^{CS(\hat{S})}$ corresponds to full search using Algorithm~\ref{alg}, $\beta^{CS}$ uses the pooled training data, $\beta^{MTL}$ performs the Multi-task feature learning algorithm~\citep{argyriou_multi-task_2006} for the MTL setting and $\beta^{DICA}$  performs DICA~\citep{muandet2013domain} for DG. For DICA, which is a nonlinear method, the kernel matrices are constructed using an rbf kernel, and the length-scale of the kernel is selected according to the median heuristic. In the MTL setting, we combine the invariance with task specific information by optimizing~\eqref{eq:obsll} using EM, resulting in regression coefficients $\beta^{CS(\hat{S}\sharp)}$ and $\beta^{CS(cau\sharp)}$ when the ground truth is known. Finally, $\beta^{CS(cau\sharp,UL)}$ indicates that unlabeled data from $T$ was also available. For reference, Figure~\ref{fig:synt_numEx_mtl} (left) provides results for $\beta^{CS(\hat{S}+)}$ and $\beta^{CS(cau+)}$, which correspond to the estimators obtained by solving the constrained optimization problem described in the paragraph \textit{`A naive estimator for comparison'} of Section~\ref{sec:MTL} ($\beta^{CS(cau+)}$ uses the ground truth for $S^*$ and $\alpha$), while $\beta^{naive}$ imputes the covariance matrices but does not optimize the free parameters. $\beta^{CS(cau+, i.d.)}$ (infinite data) also assumes that we know the ground truth for the entries of the covariance matrix for the test task corresponding to the covariance of $\B{X}$, the covariance between $\B{X}_{S^*}$ and $Y$, and the variance of $Y$. 

\subsection{Synthetic data set}
In this section, we generate a synthetic data set in which the causal structure of the problem is known.
For all experiments, we choose $\delta = 0.05$ as a rejection level for the statistical test in Algorithms~\ref{alg} and \ref{alg:greedy}. 
Moreover, we use $40\%$ of the training examples to fit the linear models in Algorithms~\ref{alg} and \ref{alg:greedy}, 
and the remaining data as validation. 
The sensitivity to the choice of $\delta$ is discussed in Section~\ref{sec:delta}. 

\paragraph{Generative process of the data}
For each task $k \in \{1, 2, \ldots, D,T\}$, we sample a set of causal variables from a multivariate Gaussian 
$$\B{X}_{S^*}^k \sim \mathcal{N}(0,\Sigma_{{S^*}}^k)$$ 
where the covariance matrix $\Sigma_{S^*}^k$ is drawn from a Wishart distribution $\C{W}(U_{S^*}^k,p)$, where 
$U_{S^*}^k$ is computed as $V^k(V^k)^t$. Here, $V^k$ is a $(|S|,|S|)$ matrix of standard Gaussian random variables. 

The target variable $Y^k$ is drawn as 
$$Y^k = \alpha \B{X}_{S^*}^k + \epsilon^k$$ 
where $\epsilon^k \sim \mathcal{N}(0,2)$ (the standard deviation of $\epsilon^k$ is $6$ for the non sparse DG experiment with $30$ predictors, see the bottom of Figure~\ref{fig:synt_numEx}). 

We sample the remaining predictor variables as 
$$\B{X}_N^k = \gamma^k Y^k + \beta^k(\B{X}_{S^*}^k)_C+\eta^k$$ 
where $\eta^k \sim \mathcal{N}(0,\Sigma_N^k)$. $(\B{X}_{S^*}^k)_C$ is a subset of $\B{X}_{S^*}^k$ of size $|C|$ which generates both the target $Y^k$ and $\B{X}_NT^k$. 
$\gamma^k$ of size $|N|$ is computed as $\gamma^k = (1-\lambda)\gamma_0 + \lambda g^k$, where $\lambda\in[0,1]$, $\gamma_0$ is the same in all tasks while $g^k$ is task dependent. 
Both $\gamma_0$ and $g^k$ are drawn from a standard Gaussian. 
Similarly to $\gamma^k$, $\beta^k$ is a $(|C|, |N|)$ matrix computed as $\beta^k = (1-\lambda)\beta_0 + \lambda b^k$.
$\Sigma_N^k$ is sampled similarly to $\Sigma_{S^*}^k$. 
Finally, $\alpha$ is sampled from a standard Gaussian distribution.

The generative process and hyper-parameters are the same for all the experiments (DG and MTL). 

\paragraph{Results}
\begin{figure}[t]
  \centering
  \hspace{-1.cm}
\begin{subfigure}[b]{0.49\linewidth}
\includegraphics[width=\textwidth]{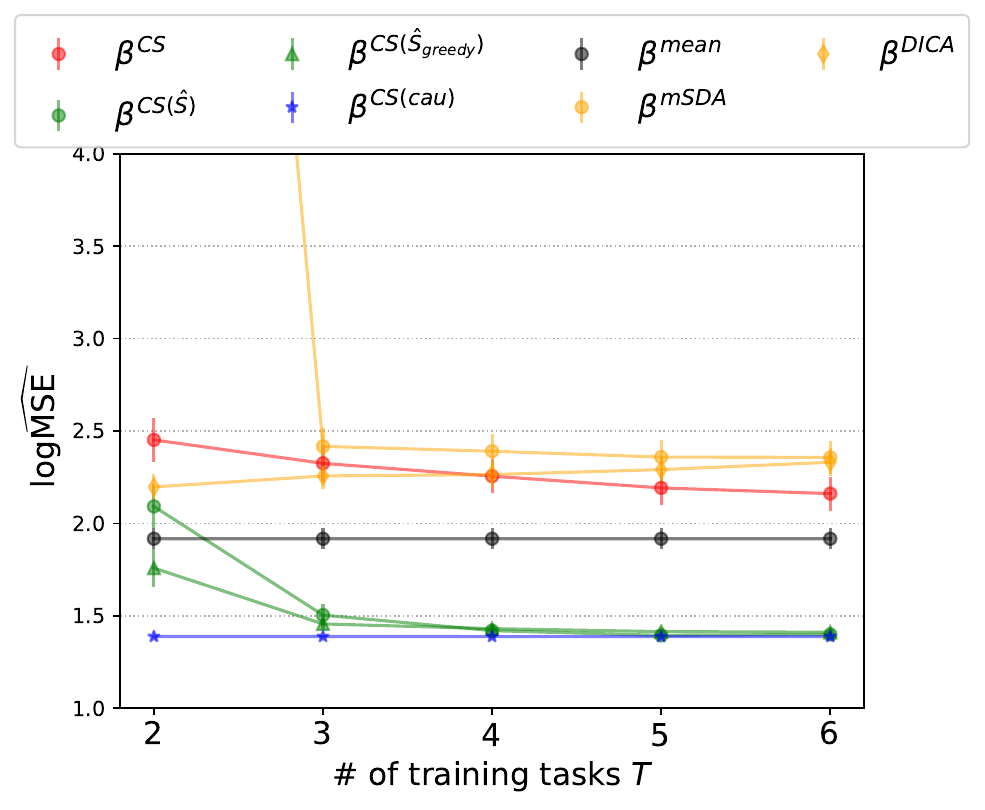}
\end{subfigure}
\hspace{0.15cm}
\begin{subfigure}[b]{0.50\linewidth}
\centering
\includegraphics[width=\textwidth]{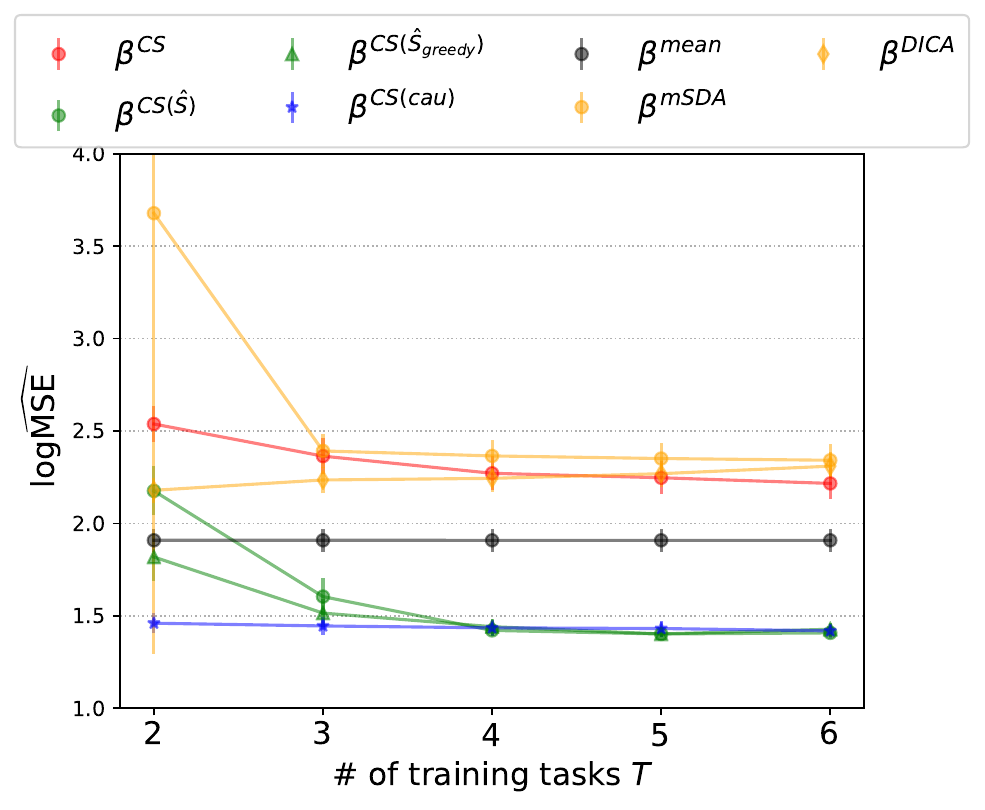}
\end{subfigure}
\hspace{0.15cm}
\begin{subfigure}[b]{0.54\linewidth}
\centering
\includegraphics[width=\textwidth]{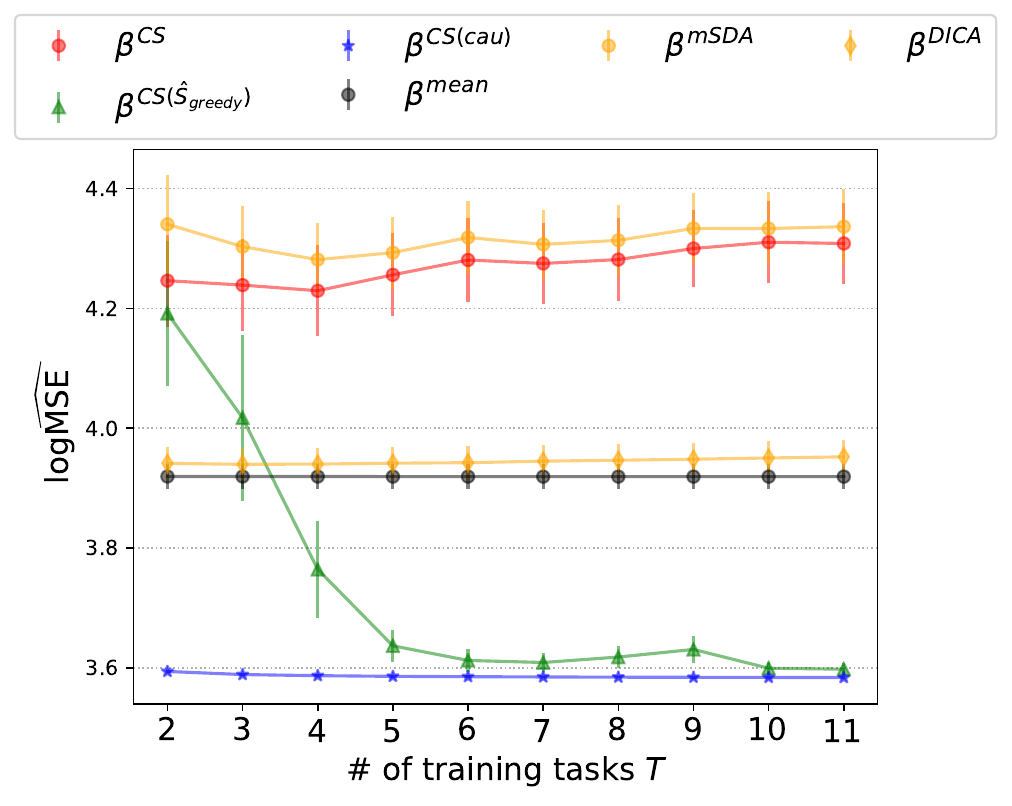}
\end{subfigure}
\caption{DG setting. Logarithm of the empirical squared error in the test task for the different estimators in the DG setting. The results show averages and $95\%$ confidence intervals for the mean performance over $100$ repetitions. 
We vary the number of tasks $D$ available at training time. 
Upper left: both $S$ and $N$ are of size~$3$, such that $\B{X}$ is $6$-dimensional. $|C|$ is of size one. 
Upper right: $30$ noise variables are added to $\B{X}$. Variable selection using the Lasso is used prior to computing $\beta^{CS(\hat{S})}$, while $\beta^{CS(\hat{S}greedy)}$ uses all predictors. 
Bottom: both $S$ and $N$ are of size~$15$. Full search is not computationally feasible in this setting 
and only the greedy procedure can be used. Other methods such as $\beta^{CS}$, $\beta^{mSDA}$ and $\beta^{DICA}$ often perform badly, which explains why in comparison $\beta^{mean}$ appears to performs well.  
}
  \label{fig:synt_numEx}
\end{figure}
\begin{figure*}[t!]
\begin{minipage}[b]{0.45\linewidth}
  \centering
  \hspace{-1.5cm}
\includegraphics[width=\textwidth]{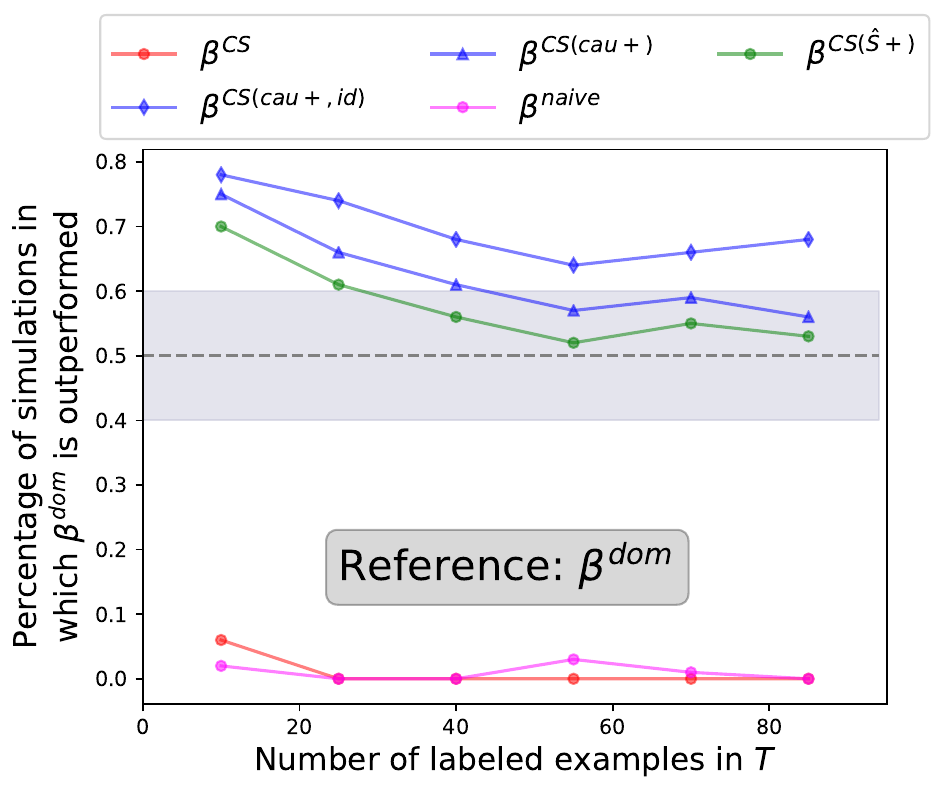}
\end{minipage}%
\hspace{-0.5cm}
\begin{minipage}[b]{0.45\linewidth}
\centering
\includegraphics[width=\textwidth]{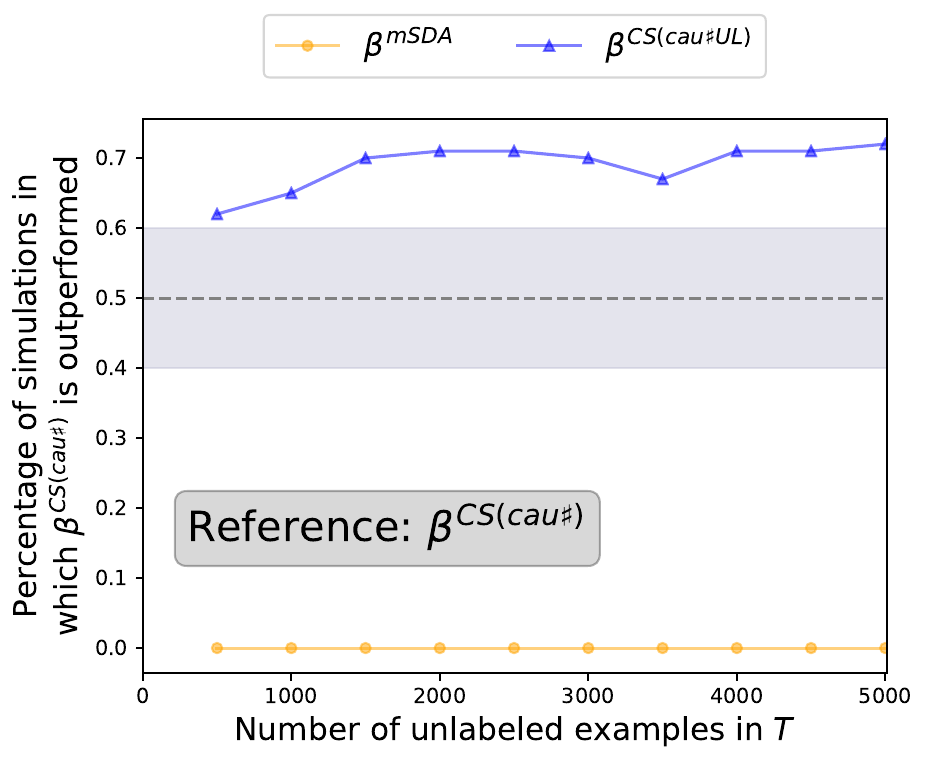}
\end{minipage}%
\hspace{-1.5cm}
\begin{minipage}[b]{0.45\linewidth}
\centering
\includegraphics[width=\textwidth]{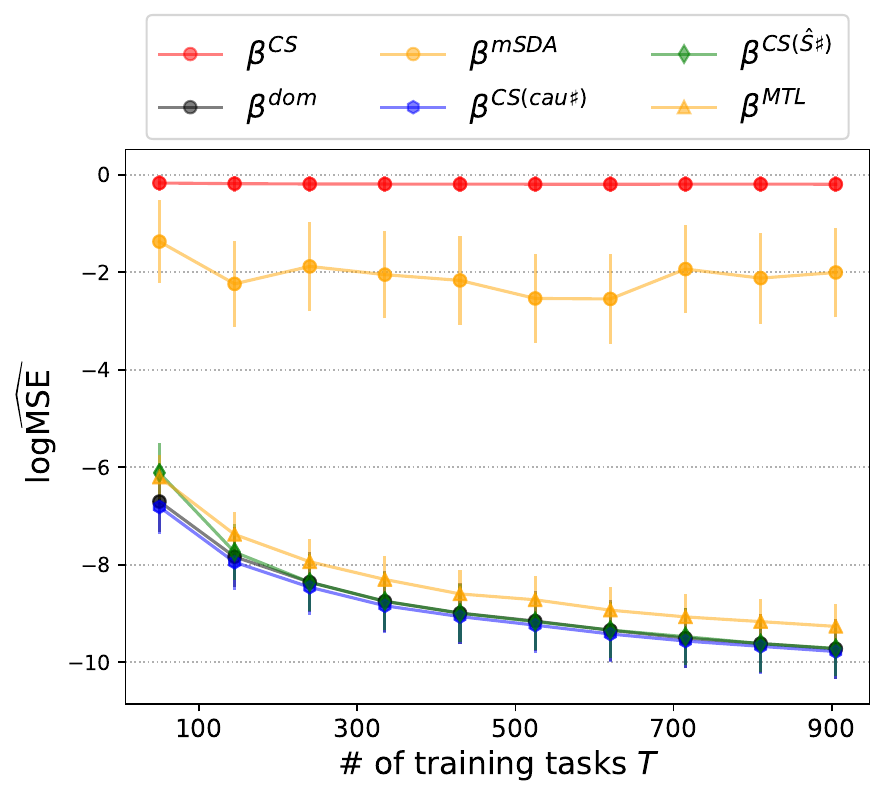}
\end{minipage}%
\hspace{0.8cm}
\begin{minipage}[b]{0.45\linewidth}
\centering
\hspace{-1.5cm}
\includegraphics[width=\textwidth]{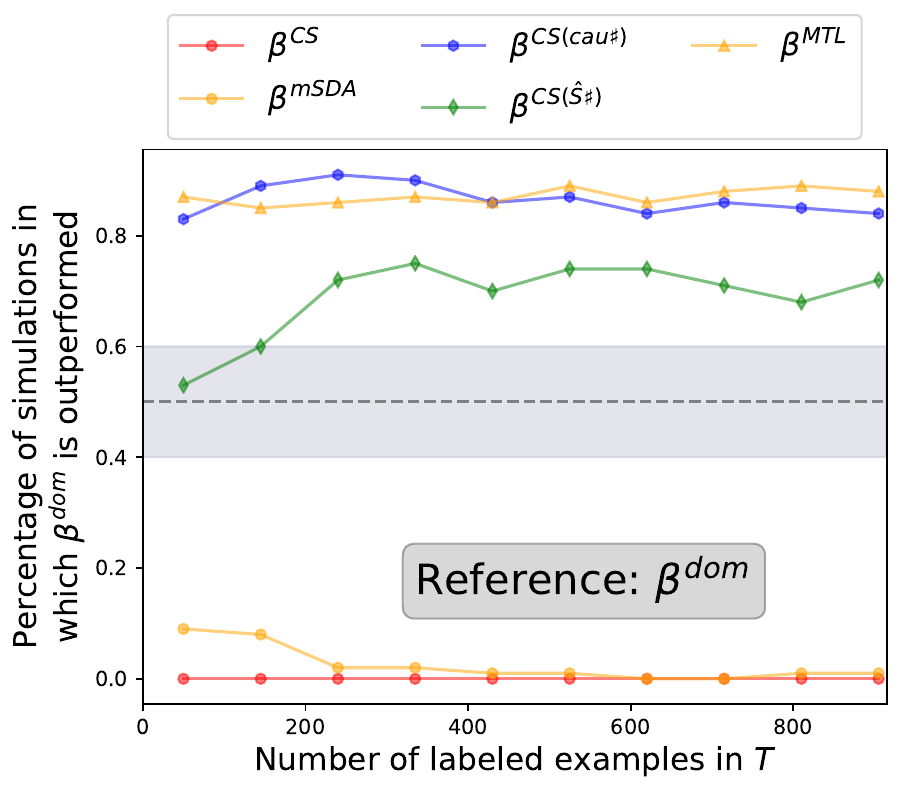}
\end{minipage}%
\caption{MTL setting. Percentage of repetitions (out of $100$) for which the corresponding method outperforms $\beta^{dom}$ (or $\beta^{CS(cau\sharp)}$ for the top right plot). Both $S$ and $N$ are of size~$3$, such that $\B{X}$ is $6$-dimensional. 
Upper left: AMTL setting. This plot shows that the methods $\beta^{CS(\hat S +)}$ and $\beta^{CS(cau +)}$ presented in Section~\ref{sec:MTL} perform well, but a large amount of data is necessary: $50000$ unlabeled examples from $T$ and $36000$ training examples are available. The naive method $\beta^{naive}$ performs poorly. 
Upper right: in the SMTL setting, we fix the number of training data ($500$ per task) and vary the amount of unlabeled data available from the test task. We report the percentage of scenarios in which the corresponding method outperforms $\beta^{CS(cau\sharp)}$ this time (which uses no unlabeled data). 
While $\beta^{mDA}$ always performs worse than $\beta^{CS(cau\sharp)}$ and does not exploit the unlabeled data, we see that $\beta^{CS(cau\sharp, UL)}$ performs better as the amount of unlabeled data increases. 
Bottom: SMTL setting, and we vary the number of labeled examples available in each training task. Here, significantly less labeled data was available in the training tasks (from $50$ to $1000$ per task). In this setting, the methods using unlabeled data were given $100$ unlabeled examples. Bottom left: logarithm of the empirical squared error in the test task for different estimators. Bottom right: percentage of repetitions (out of $100$) for which the corresponding method outperforms $\beta^{dom}$.}
  \label{fig:synt_numEx_mtl}
\end{figure*}
\begin{figure*}[t!]
\begin{minipage}[b]{0.6\linewidth}
\centering
\includegraphics[width=\textwidth]{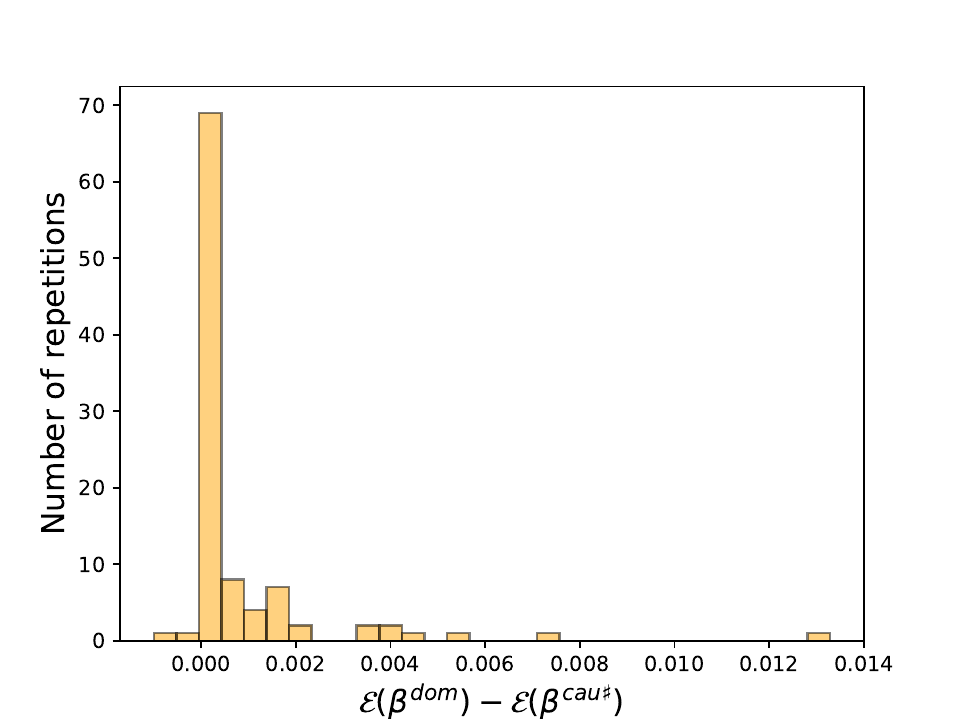}
\end{minipage}
\caption{In the SMTL setting, $900$ examples from each of the training tasks are available (this corresponds to the data point furthest to the 
right in the bottom plot of Figure~\ref{fig:synt_numEx_mtl}). 
We run $100$ repetitions and plot the histograms of $\Delta=\mathcal{E}(\beta^{dom})-\mathcal{E}(\beta^{CS(cau\sharp}))$. 
The proposed estimator outperform $\beta^{dom}$: for a large proportion of the repetitions, $\Delta>0$. More importantly, 
the distribution of $\Delta$ is heavily skewed in the positive values. In other words, when $\beta^{dom}$ outperforms 
$\beta^{CS(cau\sharp)}$, the difference in performance is small, while the difference is often larger for the converse. 
}
\label{fig:amtl}
\end{figure*}

\begin{figure*}[t!]
  \begin{minipage}[c]{0.48\linewidth}
  \centering
  \includegraphics[width=\textwidth]{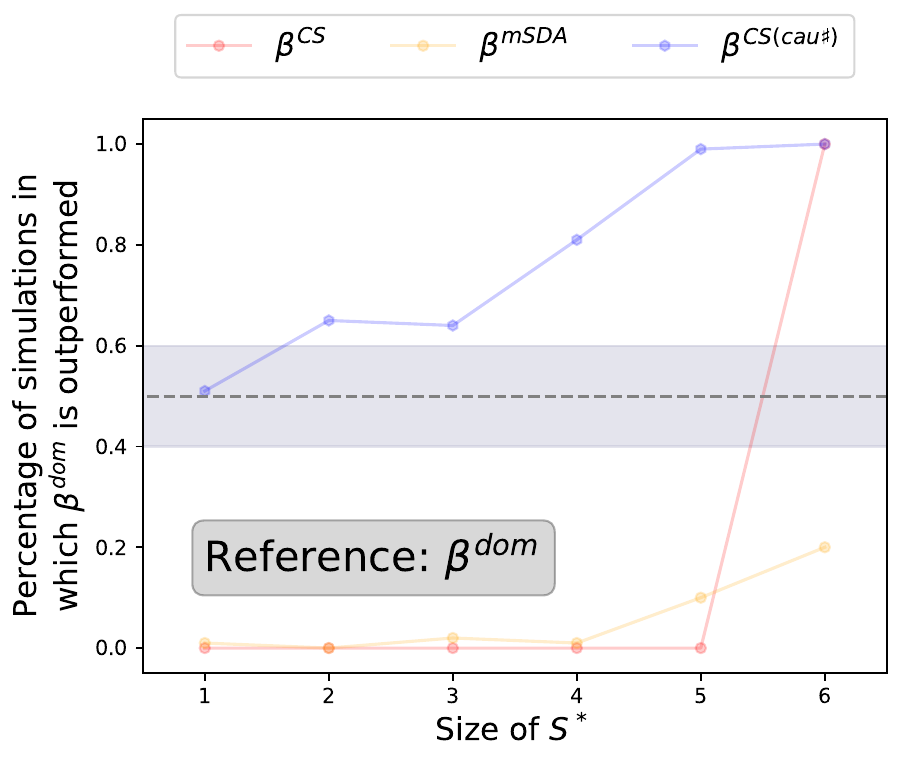}
  \end{minipage}
  \begin{minipage}[c]{0.48\linewidth}
  \centering
  \includegraphics[width=\textwidth]{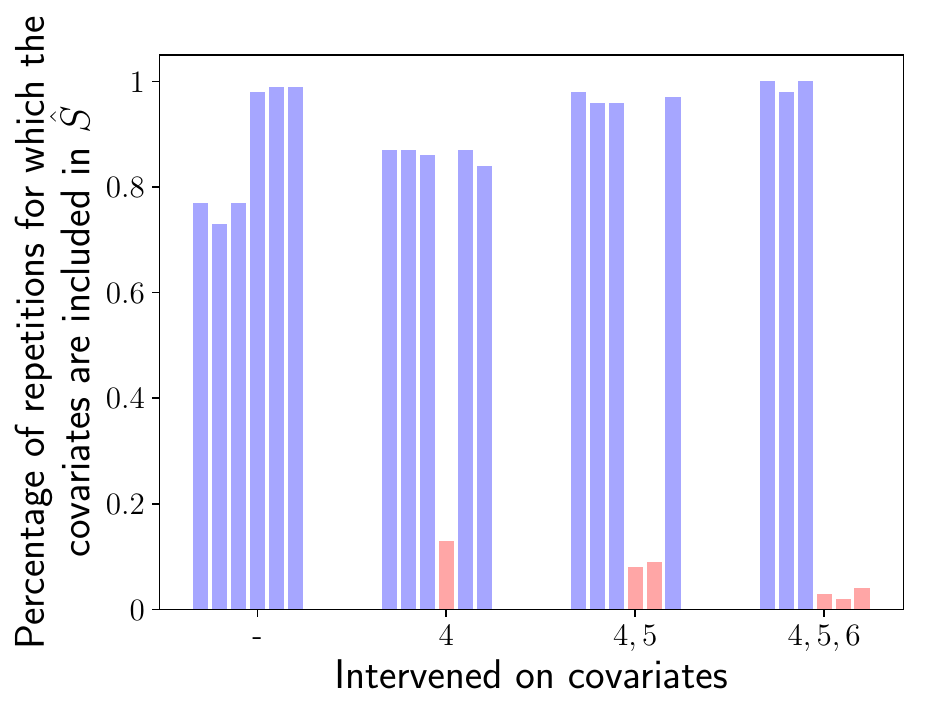}
  \end{minipage}
  \caption{
  \emph{Left}: SMTL setting with $6$ tasks and $900$ examples per task. We plot the percentage of repetitions (over $100$) for which 
  the given methods outperform $\beta^{dom}$, as a function of the size of the invariant set $S^*$. We see that as $S^*$ becomes larger, 
  more information is transferred from the training tasks, and as such the performance of $\beta^{CS(cau\sharp)}$ improves. When $S^*$ is 
  the full set, our method behaves like pooling the data. 
  \emph{Right}: Covariates selected by Algorithm~\ref{alg} when the training tasks contain interventions only on some of the covariates. 
  The bars represent the percentage of repetitions (out of $100$) for which the corresponding covariates were 
  selected. 
  When there are no interventions in the training tasks, meaning that all the training tasks follow the same distribution, 
  Algorithm~\ref{alg} systematically selects \emph{all} covariates for prediction. When more interventions are performed, however, 
  the corresponding covariates (in red) are excluded in a large number of the repetitions. 
  }
\label{fig:estimate_intervene}
\end{figure*}

\begin{figure*}[t!]
  \centering
  \includegraphics[width=.55\textwidth]{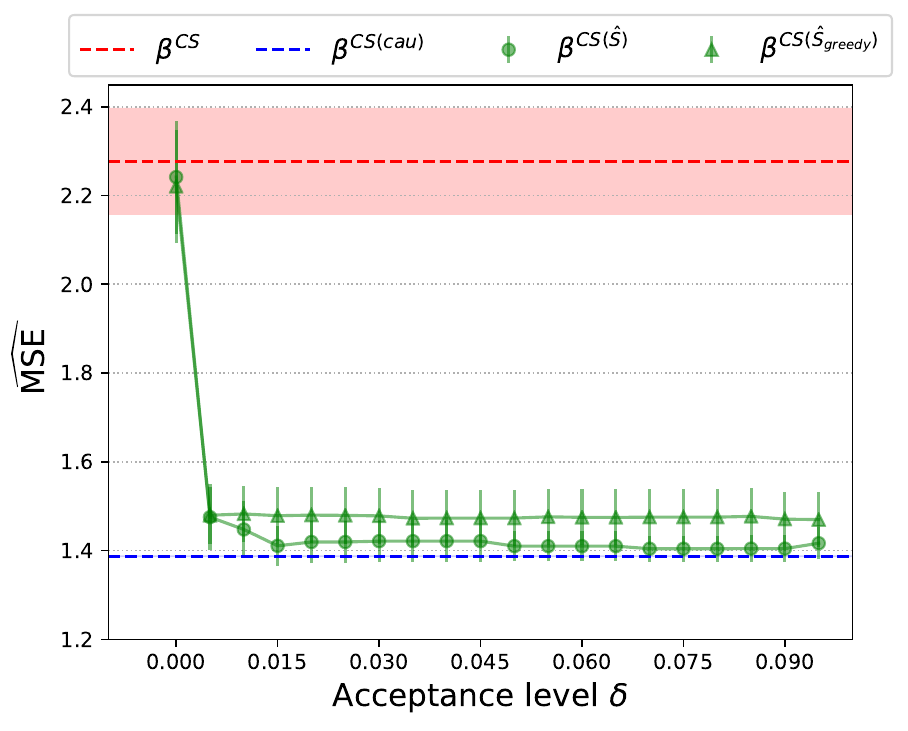}
  \caption{Logarithm of the empirical squared error in the test task in the DG setting as a function of the acceptance level of the statistical test $\delta$ in Algorithm~\ref{alg}. The setup corresponds to $t=3$ in Figure~\ref{fig:synt_numEx} (left), also over $100$ repetitions. For $\delta = 0$, all subsets are accepted, so the full set of predictors, which minimizes the validation squared error, is selected. Algorithm~\ref{alg} then returns $\beta^{CS}$. As $\delta$ increases, no subset is accepted, and Algorithm~\ref{alg} returns the subset with the largest p-value.}
\label{fig:pvals}
\end{figure*}
Our goal is to linearly predict target $Y^T$ using predictors $\B{X}^T = (\B{X}_{S^*}^T,\B{X}_N^T)$ on the test task. Given regression coefficient $\beta$, we measure the performance in the test task using the logarithm of the empirical estimator of $\C{E}_{\BB{P}^T}(\beta)$.

In Figure~\ref{fig:synt_numEx}, we are in the DG setting (thus, no labeled examples from $T$ are observed at training). $4000$ examples per training task are available for the top left and right plots, while 
only $1000$ examples per task are available on the bottom because of computational reasons. 
We report the log average empirical MSE over left out test tasks. We study both sparse and non sparse settings (in which full search is not feasible). 
On the upper left and upper right, we see that when more than four training tasks are available, 
both the full search and greedy approaches are able to recover an invariant set, and 
outperform pooling the data for any number of training tasks. When more than five training tasks are observed, $\beta^{CS(\hat{S})}$ performs like $\beta^{CS(cau)}$, which uses knowledge of the ground truth. On the bottom, full search is not feasible, and $\beta^{CS(\hat{S}greedy)}$ 
outperforms other approaches.  

In Figure~\ref{fig:synt_numEx_mtl} (top left), we consider an AMTL setting, in which large amounts of labeled data ($36000$) from the training tasks and unlabeled data from the test task ($50000$) are available. Both $S$ and $N$ are of size~$3$, such that $\B{X}$ is $6$-dimensional. For all MTL experiments, $6$ training tasks are available. 
We report the percentage of simulations for which the population MSE of a given approach outperforms $\beta^{dom}$. We see that $\beta^{CS(cau+, i.d.)}$ systematically outperforms $\beta^{dom}$. Moreover, $\beta^{CS(cau+)}$ and $\beta^{CS(\hat{S}+)}$
also perform well, and positive transfer is effective. However, a prohibitively large amount of labeled and unlabeled data is needed for these approaches, and the differences become non-significant 
for all methods except $\beta^{CS(cau+,i.d.)}$. 
  This shows the limitation of this family of approaches.  
In a setting with only $900$ examples per training task in SMTL, we plot in Figure~\ref{fig:amtl} the histogram of the error difference $\Delta=\mathcal{E}(\beta^{dom})-\mathcal{E}(\beta)$ for $\beta^{CS(cau\sharp)}$. Figure~\ref{fig:synt_numEx_mtl} (top right) corresponds to the same setting, but we vary the number of unlabeled data available (we only plot methods that use unlabeled data, and $\beta^{CS(cau\sharp)}$ is used as reference instead of $\beta^{dom}$). In Figure~\ref{fig:synt_numEx_mtl} (bottom) we consider an SMTL setting in which only $100$ unlabeled data points are available, and only few labeled examples are available in each task. Here, we see that $\beta^{CS(cau\sharp)}$, $\beta^{CS(\hat S \sharp)}$ and $\beta^{MTL}$ perform well, while other methods do not. In terms of MSE (bottom left), the difference in performance between the top competing methods is not statistically significant. 

\paragraph{Time complexity} The most expensive component of our method is the estimation of the invariant subset. 
In the DG experiment in Figure~\ref{fig:synt_numEx}, with $n=4000$ examples available for each of the $6$ tasks, and $p=6$ predictors, full subset search takes $0.067$ seconds and greedy search $0.037$, where the results are averaged over $100$ repetitions. With $p=10$, full search averages at $1.57$ seconds, and greedy search $0.0396$. With $p=30$, where full search is not feasible, greedy search averages at $1.21$ seconds. In the MTL experiment in Figure~\ref{fig:synt_numEx_mtl}, the EM algorithm runs for $0.00105$ seconds on average over $100$ repetitions. As a reference, in MTL, linear regression averages at $0.000301$ seconds and mSDA at $0.0547$ seconds.

\subsection{Sensitivity to the acceptance level $\delta$}\label{sec:delta}
Both Algorithm~\ref{alg} and its greedy version Algorithm~\ref{alg:greedy} receive an acceptance 
level $\delta$ as input for the statistical test. In our other experiments, we chose the 
standard value of $\delta = 0.05$. Figure~\ref{fig:pvals} shows 
the error on the test tasks in the DG setting for both methods for different values of $\delta$. 
The setting is the same as in the left of Figure~\ref{fig:synt_numEx} for three training tasks. 
$\beta^{CS}$ and $\beta^{CS(cau)}$ are provided as reference. For $\delta=0$, all subsets are accepted 
as invariant, thus both methods behave like pooling the data. After a critical value of $\delta$, no 
subset is accepted, and both algorithms return the subset with the largest p-value.

\subsection{Informativeness and subset estimation}\label{sec:inform}
The estimation of an invariant subset involves finding a subset for which the residuals have the same distribution across tasks. 
It is desirable, however, that the selected subset is one which explains the data best. This is ensured by selecting the subset which 
leads to the smallest error on a validation set. Therefore, some covariates in $N$ may be included in a selected subset \emph{if there are no 
interventions on this covariates in the training tasks}. More precisely, if including a covariate does not lead to a statistically measurable 
difference in the distribution of the residuals between the training tasks, it is advantageous in general to include it in the selected subset 
since the data is better explained.

We illustrate this in Figure~\ref{fig:estimate_intervene} (right) in the setting previously described with $p=6$. We estimate an invariant subset using Algorithm~\ref{alg} over $100$ 
repetitions in the following scenarios: i) all the covariates have the same distribution across tasks, ii) one, two or three covariates in $N$ are subject 
to interventions between the tasks. Figure~\ref{fig:estimate_intervene} (right) show the proportion of repetitions for which each covariate is included in the selected subset. We see that, 
as expected, covariates in $N$ for which there are no interventions are included in the selected subset in a large portion of the repetitions, while the 
other covariates are excluded. This highlights that Algorithm~\ref{alg} can only exclude covariates whose distribution shifts between training tasks. If 
being conservative is important for the problem at hand, one can modify Algorithm~\ref{alg} accordingly, see the end of Section~\ref{sec:cvMTL}.

Moreover, in Figure~\ref{fig:estimate_intervene} (left) we consider a similar setting, and we compute the performance against $\beta^{dom}$ in 
an SMTL setting as the size of the invariant set increases. We see that as the size of the invariant set increases, the performance of 
$\beta^{CS(cau\sharp)}$ improves, since more information is being transferred from the training tasks. When $p=6$, traditional covariate shift 
holds, and $\beta^{CS(cau\sharp)}$ performs on par with $\beta^{pool}$.

\subsection{Gene perturbation experiment}\label{sec:genes}

We apply our method to gene perturbation data provided by~\citet{kemmeren2014large}. This data set consists of the m-RNA expression levels of $p= 6170$ genes $X_1,\ldots,X_p$ of the Saccharomyces cerevisiae (yeast). It contains both $n_{obs} = 160$ observational data points and $n_{int} = 1479$ data points from intervention experiments. In each of these interventions, {one known gene (out of $p$ genes) is deleted}. In the following, we consider two different tasks. The observational sample is drawn from the first task, and the pooled $n_{int}$ interventions are drawn from the second task. 

\paragraph{Motivation}
\begin{figure*}[ht]
\begin{minipage}[b]{0.45\linewidth}
\centering
\includegraphics[width=6cm,height=5cm,keepaspectratio]{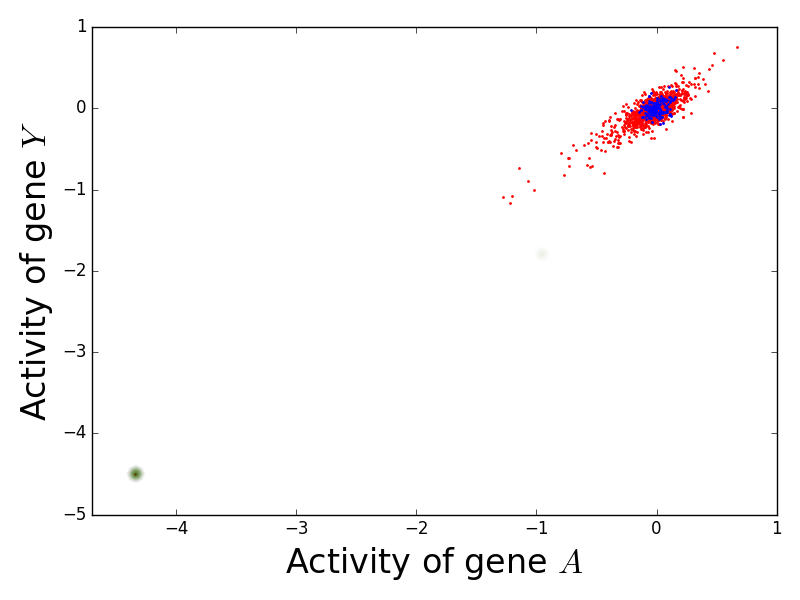}
\end{minipage}
\hspace{0.5cm}
\begin{minipage}[b]{0.45\linewidth}
\centering
\includegraphics[width=6cm,height=5cm,keepaspectratio]{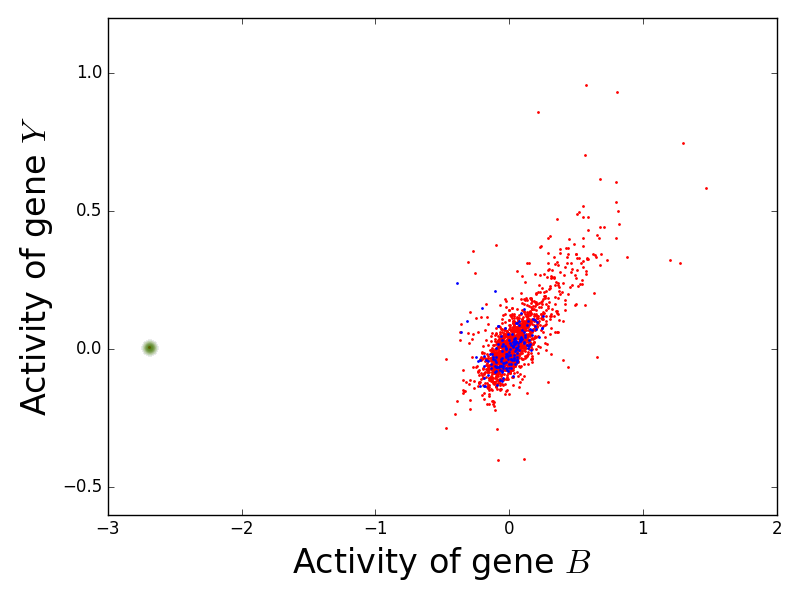}
\end{minipage}
\hspace{0.5cm}
\caption{
Example of the expression of pairs of genes, where $A$ is causal (left) and $B$ is non-causal (right) of target $Y$. The blue points are from the observational sample (task $1$), the red dots are the interventional sample (task $2$), and the green point corresponds to the single interventions in which $A$ and $B$ are intervened on respectively. On the left, a model learned on the data in red and blue would still perform well on the intervention point, which is not the case on the right.}
\label{fig:genes_motiv}
\end{figure*}
In order to gain an intuition about the experiments we are presenting, consider Figure~\ref{fig:genes_motiv}. We select as a target a gene $Y$ out of the $p$ genes, and our goal is to predict the activity of $Y$ given the remaining $p-1$ genes as features. Some of these $p-1$ genes are causal of the activation of $Y$. For example, Figure~\ref{fig:genes_motiv} shows on the x-axis the activity of two genes (gene $A$ on the left, gene $B$ on the right) such that:
\begin{compactitem}
\item The expressions of $A$ and $B$ are strongly correlated with the expression of $Y$.
\item $A$ is causal of $Y$ (here, we use the definition of a causal effect proposed by \citet{peters_causal_2015}).
\item $B$ is non-causal of $Y$ (anticausal or confounded).
\end{compactitem}
In Figure~\ref{fig:genes_motiv} (left), the blue points correspond to the $160$ data points from the {observational sample}, which corresponds to the first task. The red dots are the $1478$ data points from the {interventional sample}, except for the single data point for which $A$ is intervened on, and constitute the second task.  
The plot on Figure~\ref{fig:genes_motiv} (right) is constructed analogously for $B$. We can indeed see that in the pooled sample from task $1$ and $2$, $A$ and $B$ are both strongly correlated with target $Y$.

The key difference between both plots are the green points. On Figure~\ref{fig:genes_motiv} (left), the green dot corresponds to the single intervention experiment in which gene $A$ is intervened on. Similarly, the green dot on Figure~\ref{fig:genes_motiv} (right) is the single point in which $B$ is intervened on. Our goal is to consider the DG setting in which the test task consists on {this single intervention point}. 

For the causal gene $A$, one expects that a change in the activity of $A$ should translate into a proportional change in the activity of $Y$. We observe that, in the particular example of the left plot, a linear regression model from $A$ to $Y$ trained only on the pooled data from tasks $1$ and $2$ (blue and red in Figure~\ref{fig:genes_motiv}) 
would lead to a small prediction error on the intervened point (in green). That is, $S^* = \{A\}$ might be a good candidate for a set satisfying Assumptions~(A1), (A1') and (A2). 
For the non-causal gene $B$, however, intervening on $B$ leaves the activity of $Y$ unchanged, and the linear model learned on the data from tasks $1$ and $2$ performs badly on the test point in green. In such case, a candidate set is the empty set $S^* = \{\}$, leading to prediction using the mean of the target in the training data. A model which is aiming to test in these challenging intervention points should therefore include causal genes as features, but exclude non-causal genes. In these experiments, {we aim at testing whether we can exclude non-causal genes such as $B$ automatically}. 

\paragraph{Setup}
We address the problem of predicting the activity of a given gene from the remaining genes. We are looking at the following:
\begin{compactitem}
\item We consider $p$ different \textbf{problems}. In each problem $j\in\{1,\ldots,p\}$, we aim at predicting the activity $Y = X_j$ of gene $j$ using $(X_{\ell})_{\ell\neq j}$ as features.
 \item In each problem $j\in\{1,\ldots,p\}$, two \textbf{training tasks} $k\in\{1,2\}$ are available. The data from the first task is the observational sample, and the data from the second task are all the $n_{int}$ interventions (we shall subsequently remove some points for testing, see below).
\end{compactitem}
The goal is now to apply our method to each of the problems and estimate an invariant subset. 
Due to the large number of predictors, we first select the $10$ top predictor variables using the Lasso 
and then apply Algorithm~\ref{alg} to select a set of invariant predictors $\hat{S}$, see $\beta^{\hat{S} Lasso}$ in Table~\ref{tab:est}. We denote the indices of the features selected using Lasso by $L = (L_1,\ldots,L_{10})$.

The procedure is then evaluated as follows:
for each problem $j\in\{1,\ldots,p\}$, we first find the genes in $(X_{L_1},\ldots,X_{L_{10}})$ for which an interventional example is available. Note that this might not hold for all selected genes, since only $n_{int}<p$ interventions are available. We then iterate the following procedure (this is within the context of \emph{the same problem}): for each gene in $(X_{L_1},\ldots,X_{L_{10}})$ for which an intervention is available,
\begin{compactitem}
\item we put aside the example corresponding to this intervention 
from the training data (in the motivation example, this would correspond to the green point).
\item we estimate an invariant subset $\hat{S} \subseteq L$ using Algorithm~\ref{alg} with the remaining observational and interventional data.
\item we test all methods on the single intervention point which was put aside.
\end{compactitem}
We expect two different scenarios, as explained in the motivation paragraph above: (1) if the intervened gene is a \emph{cause} of the target gene, it should still be a good predictor
(see Section~\ref{sec:causality}); then, it should be beneficial to have this gene included in the set of predictors $\hat S$. (2) if the intervened gene is anticausal or confounded (we refer to this scenario as \emph{non-causal}), the statistical relation to the target gene might change dramatically after the intervention and therefore, one may not want to base the prediction on this gene.
In order to see this effect and understand how the different approaches for DG in Table~\ref{tab:est} handle the problem, we consider two groups of experiments.
\begin{compactitem}
\item[(1)] we select the target genes $Y$ for which one of the features in $L$ is causal for the activity of $Y$ and for which an intervention experiment is available. 
$39$ problems fall in this causal scenario. 
\item[(2)] out of the remaining problems we chose target genes with (non-causal) predictors that have been intervened on and --- in order to increase the difficulty of the problem --- that are strongly correlated with the target gene. 
  We therefore select $269$ cases for which a Pearson correlation test (the null hypothesis corresponds to no correlation) outputs a p-value equal to zero. 
\end{compactitem}
\paragraph{Results}

\begin{figure*}[ht]
    \hspace{-0.7cm}
    \begin{minipage}[b]{0.45\linewidth}
  \centering
\includegraphics[width=\textwidth]{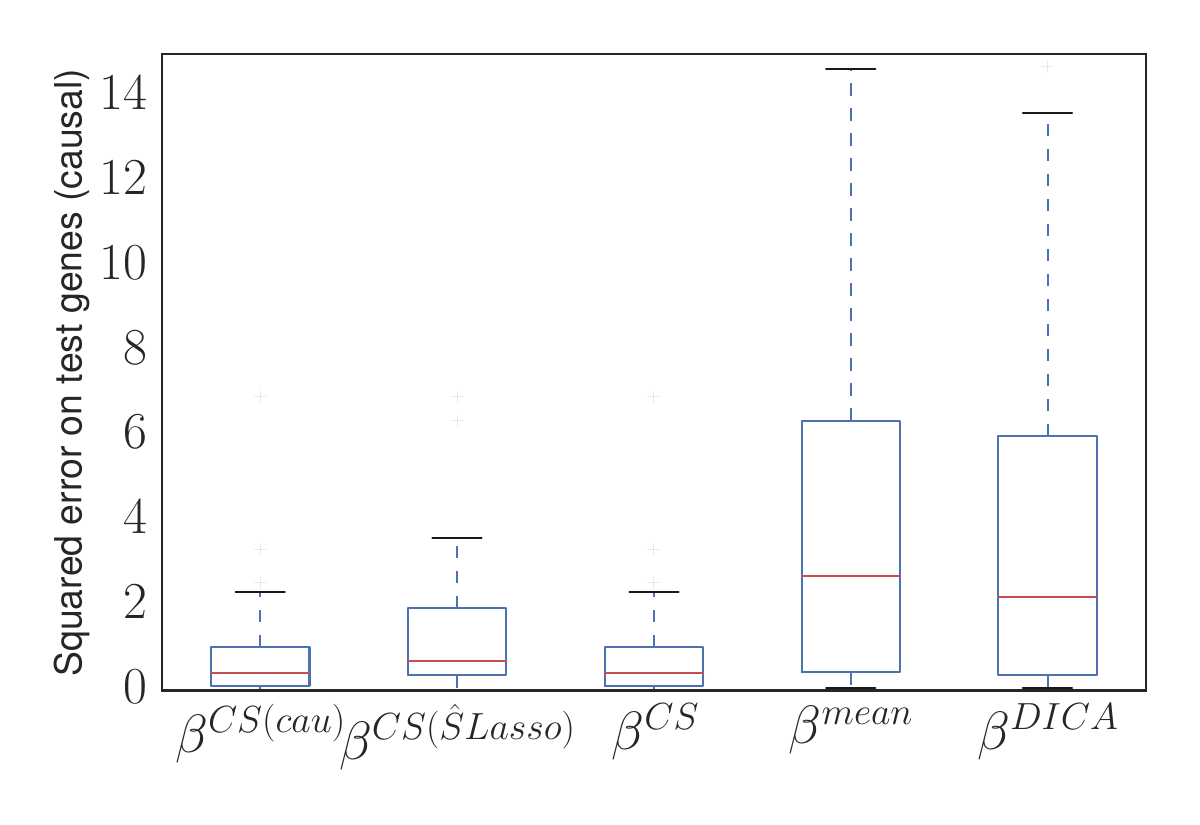}
\end{minipage}
\hspace{0.2cm}
\begin{minipage}[b]{0.45\linewidth}
\centering
\includegraphics[width=\textwidth]{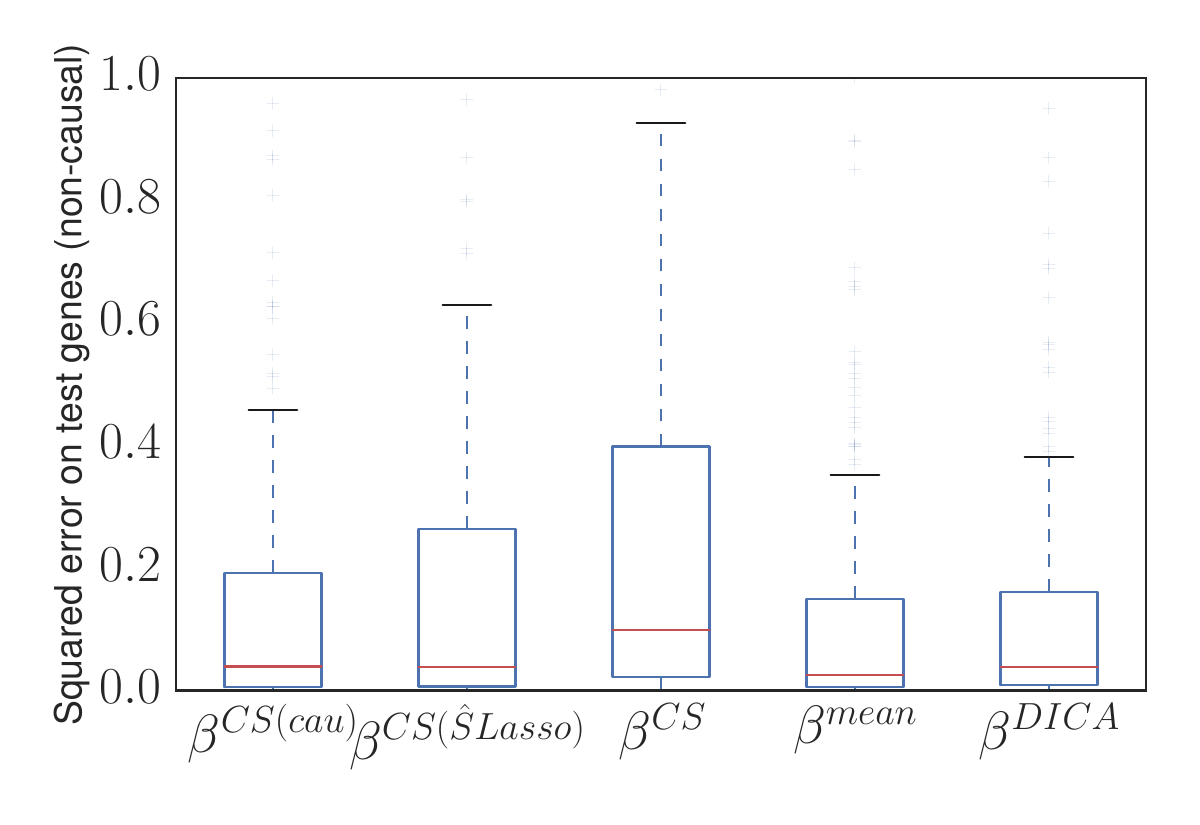}
\end{minipage}
\hspace{0.2cm}
\begin{minipage}[b]{0.45\linewidth}
\centering
\includegraphics[width=\textwidth]{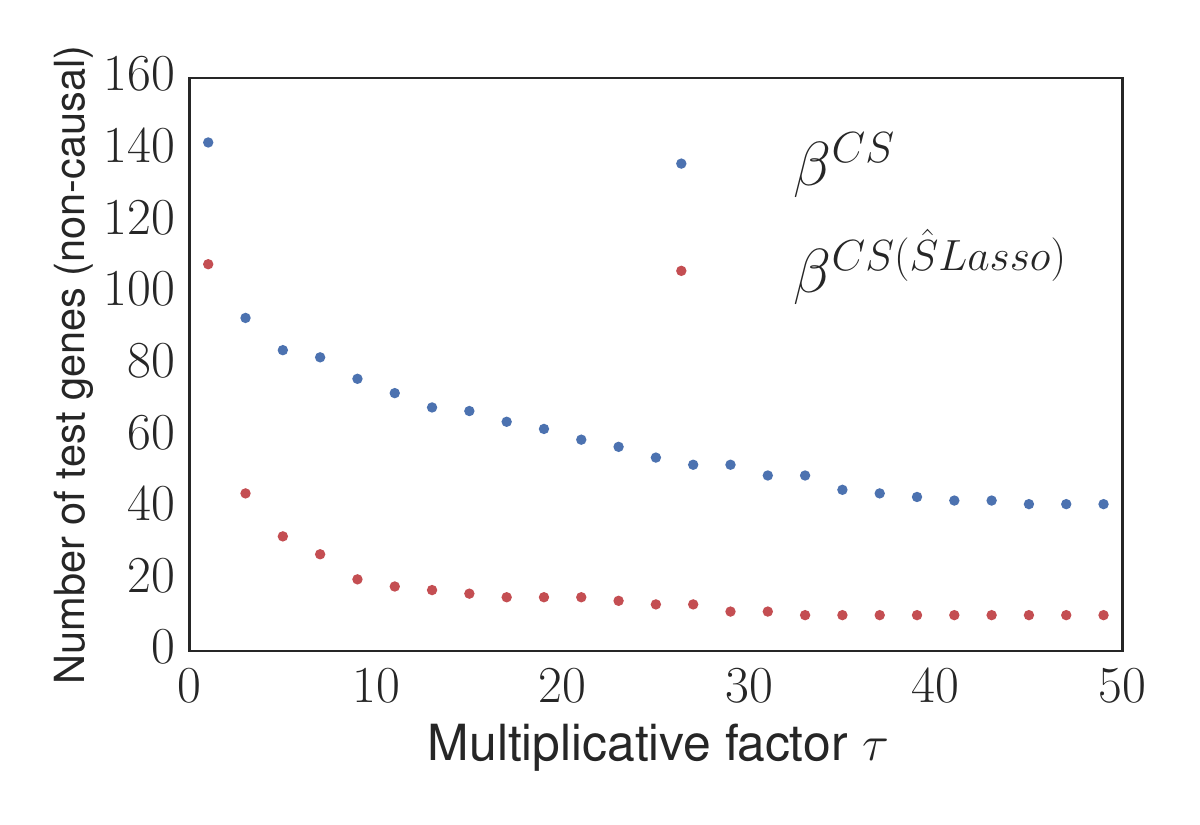}
\end{minipage}
\caption{In the causal problems (top left), interventions are performed on causal genes. As expected, the input genes continue to be good predictors, and $\beta^{CS}$ works well. In the non-causal problems (top right), one of the inputs is intervened upon and becomes a poor predictor, impairing the performance of $\beta^{CS}$. The mean predictor $\beta^{mean}$ uses none of the predictors, and therefore works comparatively well in this scenario. Our proposed estimator $\beta^{CS(\hat{S})}$ provides reasonable estimates in both the causal and non-causal settings, while other methods only perform well in one of the scenarios. $\beta^{DICA}$ performs similarly to $\beta^{mean}$ in both scenarios, and is therefore outperformed by other methods in the causal problems (note that $\beta^{DICA}$ uses all available features). Bottom: in the non-causal scenario (2), we plot the number of test genes for which the squared error for $\beta^{CS}$ is larger than $\tau$ times the squared error for $\beta^{CS(\hat{S})}$, and vice-versa, where $\tau$ is plotted on the x-axis. This plot shows the number of genes for which one of the method does significantly worse than the other. By this measure, $\beta^{CS(\hat{S}Lasso)}$ outperforms $\beta^{CS}$ for all values of $\tau$.}
\label{fig:bp_genes}
\end{figure*}
Figure~\ref{fig:bp_genes} shows box plots for the errors of the different methods for the causal problems (1) on the top left and for the non-causal problems (2) in the top right. We do not plot outliers in order to improve presentation. Figure~\ref{fig:bp_genes} (top left) presents the causal scenario. As expected, pooling does well in this setting. Figure~\ref{fig:bp_genes} (bottom) shows that in the non-causal problems (2), prediction using an invariant subset leads to less severe mistakes on test genes compared to pooling the tasks.

For comparison, since we know which predictors are being intervened on at test time, we included a method that makes use of causal knowledge: $\beta^{CS(cau)}$ uses all $10$ predictors in the causal problems (1) and all but the intervened gene in the non-causal problems (2). In practice, this causal knowledge is often not available. We regard it as promising that the fully automated procedure $\beta^{CS(\hat{S}Lasso)}$ performs comparably to $\beta^{CS(cau)}$.

\pagebreak
\section{Conclusions and further directions}{\label{sec:conclusion}}
We propose a method for transfer learning that is motivated by causal modeling and exploits a set of invariant predictors.  If the underlying causal structure is known and the tasks correspond to
interventions on variables other than the target variable, the causal parents of the target variable constitute such a
set of invariant predictors. We prove that predicting using an invariant set is optimal in an adversarial setting in DG. If the invariant structure is not known, we propose an algorithm that automatically detects an invariant subset, while also focusing on good prediction. In practice, we see that our algorithm successfully finds a set of predictors leading to invariant conditionals when enough training tasks are available. 
Our method can incorporate additional data from the test task (MTL) and yields good performance on synthetic data. 
Although an invariant set may not always exist,
our experiment on real data indicates that exploiting invariance leads to methods which are robust against transfer. 

As we saw in the DG and MTL experiments, $\beta^{\hat S}$ does not always performs as well as $\beta^{cau}$, which uses the ground truth. We believe that alternative methods for estimating the set $\hat S$ may close this gap. Furthermore, extending our framework to nonlinearities seems straight-forward and may prove to be useful in many applications. For instance, we provide a general, nonlinear version of Theorem~\ref{prop:adversarial} in Appendix~\ref{app:proofs}. Moreover, Algorithms~\ref{alg} and~\ref{alg:greedy} are presented in a linear setting. However, the extension to a nonlinear framework is straightforward. In particular, the linear regression can be replaced by a nonlinear regression method. 
We expect that there may be 
feature maps leading to invariant conditionals that are different from a subset. 

We expect our method to be favorable in (adversarial-like) situations with strong differences between the tasks, such as the gene experiment in Section~\ref{sec:genes}. We also evaluated our method on the School dataset~\citep{bakker2003task}, but found that we do not do better than pooling the data 
(we also do not do worse, the results are not shown). We believe this may be due to the fact that the difference between the tasks in this dataset are not too large. 

We believe, finally, that the link to causal assumptions and the exploitation of causal structure may lend itself well to proving additional theoretical results on transfer learning.



\newpage
\appendix
\section{}
\label{app:proofs}
In this Appendix, we provide proofs for the theoretical results in the paper, as well as an extension of Theorem~\ref{prop:adversarial}.

\subsection{A nonlinear extension of Theorem~\ref{prop:adversarial}} \label{proof:adversarial}
The extension of Theorem~\ref{prop:adversarial} to a nonlinear setting is straightforward. 
Given a subset $S^*$ leading to invariant predictions, the proposed predictor is defined as the conditional expectation
\begin{equation} \label{eq:estTL_general_app}
f_{S^*}  :\,
\begin{array}{ccc}
 \BB{R}^p & \rightarrow & \BB{R}\\
 \B{x} & \mapsto & = \BB{E}[Y^1 \given \B{X}_{S^*}^1 =\B{x}_{S^*}].
\end{array}
\end{equation}
The following theorem states that $f_{S^*}$ is optimal over the set of continuous functions $\C{C}^0$ in an adversarial setting. 
\begin{theorem} \label{prop:adversarialnonlin}
Consider $D$ tasks $(\B{X}^1, Y^1) \sim \BB{P}^1$,$\ldots$, $(\B{X}^D, Y^D) \sim \BB{P}^D$ that satisfy Assumption~(A1).
Then the estimator $f_{S^*}$ in~\eqref{eq:estTL_general_app} satisfies 
$$
f_{S^*} \in \argmin_{f\in\C{C}^0}  \sup_{\mathbb{P}^T \in \C{P}} \mean_{(\B{X}^T,Y^T) \sim \mathbb{P}^T} \left(Y^T - f(\B{X}^T)\right)^2\,,
$$
where $\C{P}$ 
contains all distributions over $(\B{X}^T, Y^T)$ that are absolutely continuous with respect to the same product measure $\mu$ and satisfy 
$Y^T \given \B{X}_{S^*}^T \overset{d}{=} Y^1 \given \B{X}_{S^*}^1$.
\end{theorem}
\begin{proof}

Consider a function $f$ that is possibly different from $f_{S^*}$, see~\eqref{eq:estTL_general_app}.
For each distribution $\mathbb{Q} \in \mathcal{P}$, we will now construct a distribution $\mathbb{P} \in \mathcal{P}$ such that 
$$
\int (y - f(\B{x}))^2 \, d\mathbb{P}
\geq
\int (y - f_{S^*}(\B{x}))^2 \, d\mathbb{Q}\,.
$$
In this proof, we assume that the probability distributions in $\mathcal{P}$ are absolutely continuous with respect to Lebesgue measure. The extension to the case where they are absolutely continuous with respect to a same product measure $\mu$ is straightforward. 
Let us therefore assume that $\mathbb{Q}$ has a density
$(\B{x}, y) \mapsto q(\B{x}, y)$. Define $\mathbb{P}$ to be the distribution that corresponds to 
$p(\B{x}, y) := q(\B{x}_{S^*}, y) \cdot q(\B{x}_{N})$,
where $\B{x}_{N}$ contains all components of $\B{x}$ that are not in ${S^*}$. In the distribution $\mathbb{P}$, the random vector $\B{X}_{N}$ is independent of $(\B{X}_{S^*}, Y)$. 
But then
\begin{align*}
&\int (y - f(\B{x}))^2 \, d\mathbb{P} \\
& = \int_{\B{x}_{N}} \int_{\B{x}_{S^*}, y} (y - f(\B{x}_{S^*}, \B{x}))^2 \, p(\B{x}_{S^*}, y) \, d\B{x}_{S^*}\, dy\, p(\B{x}_{N}) \,d\B{x}_{N}\\
&\geq \int_{\B{x}_{N}} \int_{\B{x}_{S^*}, y} (y - f_{S^*}(\B{x}_{S^*}))^2 \,p(\B{x}_{{S^*}},y)\, d\B{x}_{S^*}\, dy\, p(\B{x}_{N})\, d\B{x}_{N}\\
&= \int_{\B{x}, y} (y - f_{S^*}(\B{x}_{S^*}))^2 \,q(\B{x}_{{S^*}},\B{x}_{N},y)\, d\B{x}_{S^*}\, dy\,  d\B{x}_{N}\\
&= \int (y - f_{S^*}(\B{x}))^2 \, d\mathbb{Q}.
\end{align*}
$ $ \hfill $ $ \vspace{-1.6cm}\\
$ $ \hfill $ $

\end{proof}

\subsection{Proof of Proposition~\ref{prop:threeNodes}}\label{proof:threeNodes}
We consider three variables and the following generative process: $Y^k = \alpha^t \B{X}_{S^*}^k + \epsilon^k$, $Z^k = \gamma^kY^k + \eta^k$, where $\epsilon^k\sim\C{N}(0,\sigma^2)$, $\eta^k\sim\C{N}(0,\sigma_{\eta}^2)$ and $(\B{X}_{S^*}^k)_j \sim \C{N}(0,(\sigma_X)_j^2)$. In this model, $\gamma^k$ is the parameter responsible for the difference between the tasks, while the other parameters are shared between the tasks. 

At training time, $D$ tasks are available. We first aim to obtain an explicit formula for the linear regression coefficients $\beta^{CS} = (\beta_{S^*}^{CS}, \beta_Z^{CS})$ obtained from pooling all the training tasks together. Denote by $\B{X}$, $Y$ and $Z$ the pooled training data. For fixed $\gamma^1,\ldots,\gamma^D$, the expected loss in the training data satisfies for coefficient $\beta$ verifies:
\begin{align}\label{eq:expPool}
\BB{E} &\left(\left(Y-(\beta_X)^t \B{X} -\beta_Z Z\right)^2\right) = \frac{1}{D}\sum_{k=1}^D\BB{E}\left(Y^k-(\beta_X)^t \B{X}^k-\beta_Z Z^k\right)^2 \nonumber\\
& \qquad = \beta_X^t\mbox{diag}(\sigma_X^2)\beta_X + \frac{\beta_Z^2}{D}\left(\sigma_{\eta}^2D+V_Y\overline{\gamma^2}\right) + 2(\beta_Z\frac{\bar{\gamma}}{D}-1) \alpha^t\mbox{diag}(\sigma_X^2)\beta_X  + V_Y -2\frac{\bar{\gamma}}{D}V_Y \beta_Z,
\end{align}
where $V_Y = \alpha^t \mbox{diag}(\sigma_X^2)\alpha + \epsilon^2$.
By differentiating~\eqref{eq:expPool} with respect to $\beta$, we obtain the following expression for the pooled coefficients:
\begin{align*}
\beta_Z^{CS} = \frac{\bar{\gamma}\sigma^2}{V_Y^2\overline{\gamma^2}+D\sigma_{\eta}^2-\frac{\bar{\gamma}^2}{D}\alpha^t\mbox{diag}(\sigma_X^2)\alpha}
\quad \text{ and }
\quad \beta_{S^*}^{CS} = (1-\frac{\bar{\gamma}}{D}\beta_Z^{CS})\alpha,
\end{align*}
where $\overline{\gamma^2} = \sum_{k=1}^D (\gamma^k)^2$ and $\overline{\gamma} = \sum_{k=1}^D \gamma^k$.
Consider now an unseen test task with coefficient $\gamma^T$. The expected loss on the test task using the pooled coefficients is:
\begin{align}\label{eq:errorPoolP}
\C{E}_{\BB{P}^T}(\beta^{CS}) =\BB{E}\left((Y^T-(\beta_X^{CS})^t\B{X}^T-\beta_Z^{CS} Z^T)^2\right) &= \left(\beta_X^{CS}\right)^t\mbox{diag}(\sigma_X^2)\beta_X^{CS}+(\beta_Z^{CS})^2\left(V_Y(\gamma^T)^2 + \sigma_{\eta}^2\right)  \nonumber\\
&\qquad+ 2\beta_Z^{CS}\gamma^T\alpha^t\mbox{diag}(\sigma_X^2)\beta_X^{CS } + V_Y \nonumber\\
&\qquad - 2\alpha^t \mbox{diag}(\sigma_X^2)\beta_X^{CS}-2\beta_Z^{CS}V_Y\gamma^T.
\end{align}
 Therefore, the expectation with respect to $\gamma^T$ is:
\begin{align}
\BB{E}_{\gamma^T}\left(\C{E}_{\BB{P}^T}(\beta^{CS})\right) &= (\beta_X^{CS})^t\mbox{diag}(\sigma_X^2)\beta_X^{CS} + (\beta_Z^{CS})^2\left(V_Y\Sigma^2+\sigma_{\eta}^2\right) + V_Y - 2\alpha^t \mbox{diag}(\sigma_X^2)\beta_X^{CS} \nonumber
\end{align}
Denote by $\C{E}_{\BB{P}^T}(\beta^S) = \sigma^2$ the expected loss when using the invariant conditional predictor $\beta^{S^*} = (\alpha,0)$. Then:
\begin{align}
&\BB{E}_{\gamma^T}\left(\C{E}_{\BB{P}^T}(\beta^{CS})\right) \geq \BB{E}_{\gamma^T}\left(\C{E}_{\BB{P}^T}(\beta^{S^*})\right)\nonumber \\
&\Leftrightarrow  (\beta_X^{CS})^t\mbox{diag}(\sigma_X^2)(\beta_X^{CS}) + (\beta_Z^{CS})^2\left(V_Y\Sigma^2+\sigma_{\eta}^2\right) + V_Y-2\alpha^t \mbox{diag}(\sigma_X^2)\beta_X^{CS} \geq \sigma^2 \nonumber \\
&\Leftrightarrow (\beta_Z^{CS})^2\left(V_Y\Sigma^2+\sigma_{\eta}^2\right) \geq 2\alpha^t \mbox{diag}(\sigma_X^2)\beta_X^{CS}-(\beta_X^{CS})^t\mbox{diag}(\sigma_X^2)\beta_X^{CS}-\alpha^t\mbox{diag}(\sigma_X^2)\alpha \nonumber\\
&\Leftrightarrow (\beta_Z^{CS})^2\left(V_Y\Sigma^2+\sigma_{\eta}^2\right) \geq -\frac{\bar{\gamma}^2}{D^2}(\beta_Z^{CS})^2\alpha^t\mbox{diag}(\sigma_X^2)\alpha,
\end{align}
by replacing $\beta_X^{CS} = \alpha-\alpha\frac{\overline{\gamma}}{D}\beta_Z^{CS}$. This inequality holds true for any value of the variance $\Sigma^2$, and the pooled coefficient leads to larger error in expectation. 

Consider now that the coefficients $\gamma^k$ are fixed and centered around a non-zero value $\mu$. Then the expectation with respect to $\gamma^T$ of the loss in the test task is the following:
\begin{align}
\BB{E}_{\gamma^T}\left(\C{E}_{\BB{P}^T}(\beta^{CS})\right) &= (\beta_X^{CS})^t\mbox{diag}(\sigma_X^2)\beta_X^{CS}+(\beta_Z^{CS})^2\left(V_Y(\Sigma^2+\mu^2) + \sigma_{\eta}^2\right)  \nonumber\\
&\qquad+ 2\beta_Z^{CS}\alpha^t\mbox{diag}(\sigma_X^2)\beta_X^{CS }\mu + V_Y 
- 2\alpha^t \mbox{diag}(\sigma_X^2)\beta_X^{CS}-2\beta_Z^{CS}V_Y\mu.
\end{align}
Then, if $\bar{\gamma} \neq 0$ (if $\bar{\gamma} = 0$, both estimators coincide): 
\begin{align}
&\BB{E}_{\gamma^T}\left(\C{E}_{\BB{P}^T}(\beta^{CS})\right) \geq \BB{E}_{\gamma^T}\left(\C{E}_{\BB{P}^T}(\beta^{S^*})\right)\; \Leftrightarrow \; \Sigma^2 \geq P(\mu),
\end{align}
where $P(\mu) =- \mu^2 -\frac{2}{\beta_Z^{CS}}\left(\left(1-\frac{\bar{\gamma}}{D}\beta_Z^{CS}\right)\frac{\alpha^t\mbox{diag}(\sigma_X^2)\alpha}{V_Y}-1\right)\mu -\frac{\bar{\gamma}^2}{V_Y D^2}\alpha^t\mbox{diag}(\sigma_X^2)\alpha+\frac{\sigma_{\eta}}{V_Y}$. 

\subsection{Proof of Proposition~\ref{prop:mtlok}}
\label{proof:mtlok}
\begin{proof}
  For $k\in\{1,\ldots,D,T\}$, let $\BB{Q}^k$ be the probability distribution with density:
  \begin{equation} \label{eq:surrog}
    q^k(\B{x}_{S^*}, \B{x}_N, y) := p^k(\B{x}_{S^*}, y) p^T(\B{x}_N \given \B{x}_{S^*}, y).
  \end{equation}
  In the test task $T$, we trivially have $q^T= p^T$. First, it is easy to see that $q^k$ and $p^k$ have the same marginal distribution over $\B{x}_{S^*}$ and $y$. Indeed:
  \begin{align}\label{eq:marginal}
    \displaystyle q^k(\B{x}_{S^*},y) &= \int_{\BB{R}^{|N|}}q^k(\B{x}_{S^*}, \B{x}_N, y)d\B{x}_N \nonumber\\
    & = \int_{\BB{R}^{|N|}} p^k(\B{x}_{S^*}, y) p^T(\B{x}_N \given \B{x}_{S^*}, y)d\B{x}_N \nonumber \\
    & = p^k(\B{x}_{S^*}, y) \int_{\BB{R}^{|N|}}p^T(\B{x}_N \given \B{x}_{S^*}, y)d\B{x}_N =  p^k(\B{x}_{S^*}, y).
  \end{align}
Second, we prove that the conditional $q^k(y \given \B{x}_{S^*}, \B{x}_N)$ is the same in all tasks. Indeed, by applying Bayes' rule:

  \begin{align*}
    q^k(y \given \B{x}_{S^*}, \B{x}_N)  &= q^k(\B{x}_N \given y, \B{x}_{S^*})\frac{q^k(y, \B{x}_{S^*})}{q^k(\B{x}_{S^*}, \B{x}_N)} \\
    &= p^T(\B{x}_N \given y, \B{x}_{S^*}) \frac{q^k(y \given \B{x}_{S^*})}{q^k(\B{x}_N \given \B{x}_{S^*})} \\
    &=  p^T(\B{x}_N \given y, \B{x}_{S^*}) \frac{p^k(y \given \B{x}_{S^*})}{\int_{\BB{R}}q^k(y,\B{x}_{N} \given \B{x}_{S^*})dy} \\
    &= p^T(\B{x}_N \given y, \B{x}_{S^*}) \frac{p^k(y \given \B{x}_{S^*})}{\int_{\BB{R}}q^k(\B{x}_N \given y,\B{x}_{S^*})q^k(y \given \B{x}_{S^*})dy} \\
    &= p^T(\B{x}_N \given y, \B{x}_{S^*}) \frac{p^k(y \given \B{x}_{S^*})}{\int_{\BB{R}}p^T(\B{x}_N \given y,\B{x}_{S^*})p^k(y \given \B{x}_{S^*})dy}.
  \end{align*}
We have used the fact that $q^k(\B{x}_N \given y,\B{x}_{S^*}) = p^T(\B{x}_N \given y,\B{x}_{S^*})$, which follows from~\eqref{eq:marginal}. Since the last equality leads to a term which is equal in all tasks (indeed, Assumption (A1) ensures that $p^k(y \given \B{x}_{S^*})$ is the same for all $k\in\{1,\ldots,D,T\}$), we have the desired result. 
\end{proof}

\subsection{Statement and proof of Proposition~\ref{prop:combinelabels}}\label{sec:prop_combl}
In this Section, we provide an analytic expression for $\beta^{opt}$ from~\eqref{eq:betaopt} in terms of $\alpha$ and $\epsilon$.

\begin{proposition} \label{prop:combinelabels}
Assume that $\B{X}_{S^*}$ follows an arbitrary distribution and that Assumptions~(A1) and~(A2) hold.
Let $\gamma\in\BB{R}^{|N|}$ be the solution of an $L^2$ regression from $\B{X}_N^T$ on $Y^T$. Therefore, we can write $\B{X}_N^{T} = \gamma Y^{T} + \eta$, with $\eta$ uncorrelated to $Y^T$, and the components of $\eta$ can be correlated. 
Then the regression coefficients 
$\beta^{opt} = (\beta_{S^*}^{opt},\beta_N^{opt})$ minimizing the expected squared loss in the test task satisfy
\begin{align}
\beta_N^{opt}  &= \mathbb{E}(\epsilon^2)M^{-1}\gamma \,, \label{eq:new_betaN} \\
\beta_{S^*}^{opt}  &=   \alpha\left(1 - (\gamma^T)^t \beta_N^{T}\right)-\Sigma_{X,S^*}^{-1}\Sigma_{X,N}\beta_N\,,
\label{eq:new}
\end{align}
where $M =  \mathbb{E}(\epsilon^2)\gamma\gamma^t +\Sigma_{N} -\Sigma_{X,N}^t\Sigma_{X,S^*}^{-1}\Sigma_{X,N}$, and
$\Sigma_N := \mathbb{E}(\eta\eta^t)$, $\Sigma_{X,S^*} := \mathbb{E}(\B{X}_{S^*} \B{X}_{S^*}^t)$, $\Sigma_{X,N} := \mathbb{E}(\B{X}_{S^*}\eta^t)$ are the corresponding Gram matrices.\footnote{We dropped the superscript $T$ to lighten the notation.}
\end{proposition}

\begin{proof}  
To simplify notation, we write $Y^T$, $\B{X}_{S^*}^T$ and $\B{X}_{N}^T$ as $Y$, $\B{X}_{S^*}$ and $\B{X}_{N}$. We compute the gradients of the expected squared loss after replacing the expression for $Y$ and $\B{X}_{S^*}$:
\begin{align*}
L&=\mathbb{E}(Y-\beta_{S^*}^t\B{X}_{S^*}-\beta_N^t\B{X}_N)^2 \\
&= (\alpha(1-\gamma^t\beta_N)-\beta_{S^*})^t\Sigma_{X,{S^*}}(\alpha(1-\gamma^t\beta_N)-\beta_{S^*}) \\
&\qquad+ (1-\beta_N^t\gamma)^2\mathbb{E}(\epsilon^2)
+ \beta_N^t\Sigma_{N}\beta_N -2(\alpha(1-\gamma^t\beta_N)-\beta_{S^*})^t\Sigma_{X,N}\beta_N
\end{align*}
The gradients satisfy
\begin{align*}
\frac{\partial L}{\partial \beta_{S^*}} &= -2\Sigma_{X,{S^*}}(\alpha(1-\gamma^t\beta_N)-\beta_{S^*})+2\Sigma_{X,N}\beta_{N} \\
\frac12 \frac{\partial L}{\partial \beta_N} &= \Sigma_{N}\beta_N -(1-\gamma^t\beta_N)\mathbb{E}(\epsilon^2)\gamma + \gamma\alpha^t\Sigma_{X,N}\beta_N \\
&\qquad-\gamma\alpha^t\Sigma_{X,{S^*}}(\alpha(1-\gamma^t\beta_N)-\beta_{S^*}) 
 - \Sigma_{X,N}^t(\alpha(1-\gamma^t\beta_N)-\beta_{S^*})
\end{align*}
By setting these to zero, we find the stated values for $\beta_{S^*}^{opt}$ and $\beta_N^{opt}$.
\end{proof}

\section{}
The code to reproduce the experiments in the paper can be found in

\noindent \url{https://github.com/mrojascarulla/causal\_transfer\_learning}.

\newpage

\vskip 0.2in
\bibliography{bibfile1,bibfile2}

\end{document}